\documentclass{article}

\usepackage[preprint]{neurips_2026}

\usepackage[utf8]{inputenc} 
\usepackage[T1]{fontenc}    
\usepackage{hyperref}       
\usepackage{url}            
\usepackage{booktabs}       
\usepackage{amsfonts}       
\usepackage{nicefrac}       
\usepackage{microtype}      
\usepackage{xcolor}         
\usepackage{amsmath}
\usepackage{amssymb}
\usepackage{amsthm}
\usepackage{algorithm}
\usepackage{algpseudocode}
\usepackage{graphicx}
\usepackage{wrapfig}
\usepackage{multicol}
\usepackage{multirow}
\usepackage{subcaption}
\usepackage{float}

\newcommand{\cc}[1]{\mathcal{#1}}
\newcommand{\bb}[1]{\mathbb{#1}}

\newcommand{\mbf}[1]{\mathbf{#1}}

\newtheorem{theorem}{Theorem}
\newtheorem{corollary}{Corollary}[theorem]
\newtheorem{lemma}{Lemma}
\newtheorem{assumption}{Assumption}

\title{Adaptive Policy Learning Under Unknown Network Interference}

\author{%
  Aidan Gleich \\
  Department of Statistical Science\\
  Duke University
  \texttt{aidan.gleich@duke.edu} \\
  \And
  Eric Laber \\
  Department of Statistical Science\\
  Duke University
  \texttt{eric.laber@duke.edu} \\
  \AND
  Alexander Volfovsky \\
  Department of Statistical Science\\
  Duke University
  \texttt{alexander.volfovsky@duke.edu}   
}

\begin{document}

\maketitle

\begin{abstract}
Adaptive experimentation under unknown network interference requires solving two coupled problems: (i) learning the underlying dynamics of interference among units and (ii) using these dynamics to inform treatment allocation in order to maximize a cumulative outcome of interest (e.g. revenue). Existing adaptive experimentation methods either assume the interference network is fully known or bypass the network by operating on coarse cluster-level randomizations. We develop a Thompson sampling algorithm that jointly learns the interference network and adaptively optimizes individual-level treatment allocations via a Gibbs sampler. The algorithm returns both an optimized treatment policy and an estimate of the interference network; the latter supports downstream causal analyses such as estimation of direct, indirect, and total treatment effects. For additive spillover models, we show that total reward is linear in the treatment vector with coefficients given by an $n$-dimensional latent score. We prove a Bayesian regret bound of order $\sqrt{nT \cdot B \log(en/B)}$ for exact posterior sampling; empirically, our Gibbs-based approximate sampler achieves regret consistent with this rate and remains sublinear when the additive spillovers assumption is violated. For general Neighborhood Interference, where this reduction is unavailable, we analyze an explore-then-commit variant with $O(n^2 \log T)$ graph-discovery cost. An information-theoretic $\Omega(n \log T)$ lower bound complements both results. Empirically, our method achieves more than an order-of-magnitude reduction in regret in head-to-head comparisons. On two real-world networks, the algorithm achieves sublinear regret and yields downstream effect estimates with small RMSE relative to the truth.
\end{abstract}

\section{Introduction}

\begin{figure}[ht]
\centering
\includegraphics[width=0.98\textwidth]{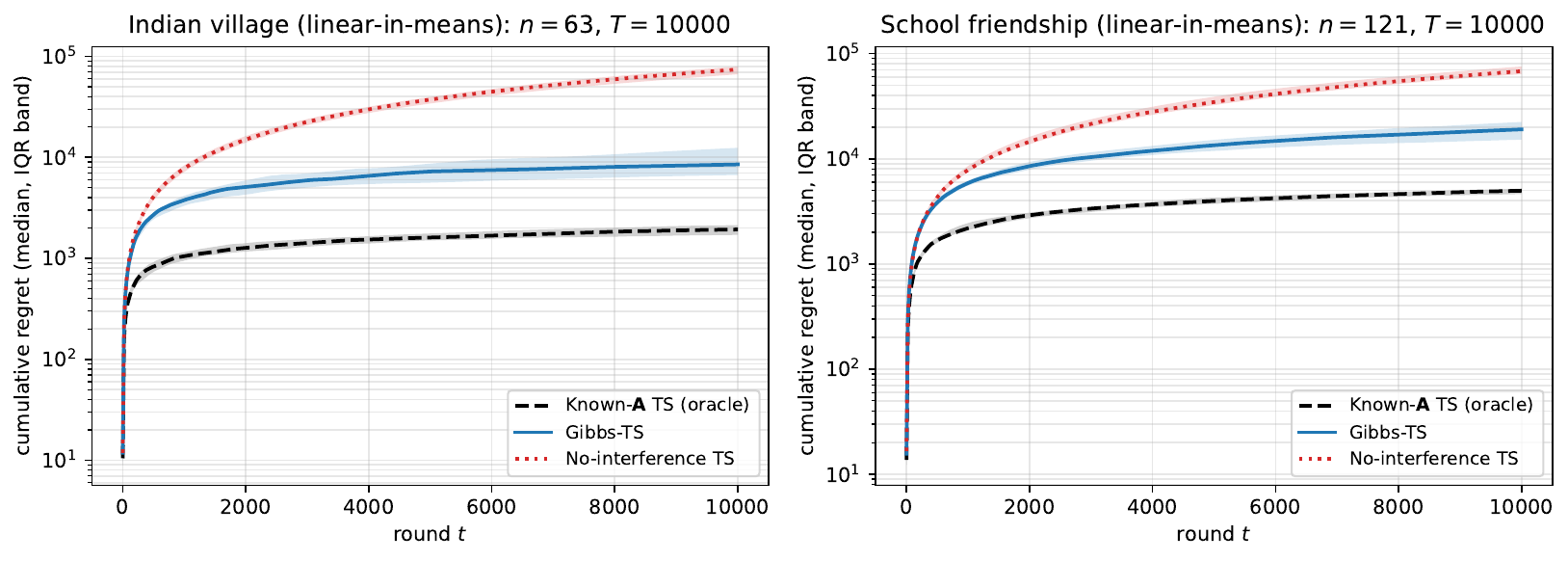}
\caption{Log cumulative regret on two real-world networks: village~10 borrowing-money household network from \citet{banerjee_diffusion2013} ($n=63$, left) and school SCHID~3 friendship network from \citet{paluck_shepherd_aranow2016} ($n=121$, right). Gibbs-TS achieves sublinear regret on both networks at scales where prior adaptive methods fail to run; the gap to the Known-$\mathbf{A}$ oracle reflects the cost of network learning. Ignoring interference (no-interference TS) yields linear regret. }
\label{fig:real_data}
\end{figure}

In many experimental settings, the outcome of one unit depends on the treatments assigned to others. Vaccination reduces disease transmission to untreated contacts \citep{laber2018,su2019, hudgens_holloran}, discounts on an online marketplace affect the revenue of competing goods \citep{wager_xu}, and educational interventions propagate through peer networks. This phenomenon, known as interference, complicates static and sequential experimental design: the optimal treatment for any given unit depends on which of its neighbors are treated.

The causal inference literature has developed a substantial apparatus for experimentation under interference using static experiments or observational data.  \citep{sussman2017, aronow_samii, belloni, ugander_gcr, EcklesKarrerUgander, jagadeesan2020designs}.  Yet in many of the applications that motivate this literature, treatments are deployed sequentially and the experimenter can update allocations as outcomes accrue. A platform running an A/B test observes daily engagement and can reallocate treatment across users; a public health agency rolling out a vaccine can adjust coverage geographically as uptake data arrives.

 Methods for the sequential problem have been proposed but are limited by scale or context. \citet{agarwal2024} propose a method for online experimentation under general neighborhood interference but rely on Fourier analysis over all possible neighbor configurations, producing a parameter count that is exponential in neighborhood size; their method does not scale beyond roughly a dozen units. \citet{viviano_rudder2020} propose an adaptive experimental design for optimal policy estimation under unknown interference but require a predefined cluster structure and do not recover the interference network itself. \citet{gleich2026scalable} develop a scalable Thompson sampling on networks of hundreds of nodes but assume the network is known. None of these methods jointly learn the  network and optimize individual-level treatment allocations in a sequential setting.

The learned interference network is itself of scientific interest: it supports estimation of direct and indirect effects \citep{on_causal_interference, toulis13}, design of follow-up experiments \citep{jagadeesan2020designs, aronow_samii, basse2018}, and identification of the most influential units. 
In many applications, the network is as valuable as the treatment policy, yet no existing adaptive method explicitly recovers it.

 We propose an adaptive experimental design that jointly learns the interference network and adaptively optimizes in a sequential experiment. Our contributions are as follows.

\begin{enumerate}
    \item \textbf{A framework for adaptive policy learning under unknown interference.} We formulate the joint problem of network recovery and adaptive treatment allocation under the standard Neighborhood Interference Assumption (NIA), and develop a Thompson sampling algorithm that solves it. Joint posterior samples over the network and reward parameters are drawn by edge-wise Gibbs sweeps.
    The algorithm returns both an optimized treatment policy and an estimate of the interference network.
    \item \textbf{Downstream causal inference from the recovered network.} The estimated $\hat{\mathbf{A}}$ supports post-experiment estimation of direct, indirect, and total treatment effects, design of follow-up randomized experiments, and identification of highly connected units. We empirically validate this in Section~\ref{sec:exp_downstream}, showing that effect estimates obtained under $\hat{\mathbf{A}}$ closely track those obtained under the true network.
    \item \textbf{Regret bounds and graph-discovery limits.} For additive spillover models, we show that total reward is linear in the treatment vector with $n$-dimensional latent score coefficients, and prove a Bayesian regret bound of order $\sqrt{nT \cdot B \log(en/B)}$ for exact posterior sampling. For general Neighborhood Interference rewards, we analyze an explore-then-commit variant with $O(n^2 \log T)$ graph-discovery cost. An information-theoretic $\Omega(n \log T)$ lower bound applies to any consistent algorithm under general NIA.
    \item \textbf{Empirical results at scale.}
    On two real-world networks, the algorithm achieves sublinear regret (Figure~\ref{fig:real_data}) under the linear-in-means specification \citep{BRAMOULLE200941, annurev_survey}, and the recovered network supports downstream estimation of direct, indirect, and total treatment effects. On small networks where \citet{agarwal2024} is feasible, our algorithm achieves regret an order of magnitude lower.  Under general interference, our method scales to networks $3\times$ the size handled by existing methods and scales up to $n=1000$ under linear-in-means rewards.
\end{enumerate}

The remainder of the paper is organized as follows. Section~\ref{sec:related} situates our work within the causal inference and adaptive experimentation literatures. Section~\ref{sec:model} introduces the sequential experimental setup and the reward model implied by the Neighborhood Interference Assumption. Section~\ref{sec:gibbs_ts} develops the Thompson sampling algorithm and its edge-wise Gibbs sampler. Section~\ref{sec:theory} presents the regret analysis of the explore-then-commit variant and the lower bound. Section~\ref{sec:experiments} reports empirical results, and Section~\ref{sec:discussion} concludes.

\section{Background and Related Work}
\label{sec:related}

\paragraph{Interference and experimental design.}
A fundamental challenge in causal inference on networks is that the outcome of one individual may depend on the treatment assigned to others, violating the stable unit treatment value assumption \citep{rubin1978}. This phenomenon, known as interference or spillover, arises in settings ranging from vaccine trials \citep{hudgens_holloran} to online marketplaces \citep{wager_xu}. A substantial literature addresses interference in the design and analysis of experiments, typically by restricting its structure through assumptions such as partial interference \citep{hudgens_holloran}, neighborhood interference \citep{sussman2017, belloni}, and exposure mappings \citep{aronow_samii}. Under these assumptions, researchers have developed methods for estimating causal effects \citep{on_causal_interference, toulis13, aronow_samii}, designing randomized experiments \citep{hudgens_holloran, EcklesKarrerUgander, ugander_gcr}, and conducting inference on peer influence parameters \citep{BRAMOULLE200941, annurev_survey}. This work is almost exclusively in static settings: treatments are assigned once according to a fixed design, and the goal is estimation or inference after the experiment concludes.

\paragraph{Learning interference structure.}
The question of whether and how individuals interfere with one another is often of independent scientific interest. In the peer effects literature, the linear-in-means model \citep{manski_ident, BRAMOULLE200941} parameterizes interference as a function of the average treatment among neighbors, and estimation of the peer effect parameter has been widely studied \citep{annurev_survey}. More generally, identifying the interference network from experimental data enables downstream causal tasks: estimating direct and indirect treatment effects, designing future experiments with appropriate clustering \citep{ugander_gcr, EcklesKarrerUgander}, and understanding the mechanisms through which treatments propagate. Our work contributes to this line of research by jointly learning the interference network and optimizing treatment allocations in a sequential setting.

\paragraph{Adaptive experimentation under interference.}
In many applications, treatments must be assigned sequentially and the experimenter can adapt allocations based on observed outcomes. The multi-armed bandit framework provides a natural formulation for such problems \citep{Lattimore_Szepesvári_2020}, but standard algorithms assume that the outcome of each unit depends only on its own treatment. \citet{jia2024multiarmedbanditsinterference} study adaptive experimentation on spatial grids with decaying interference, comparing against a restricted class of constant policies. \citet{agarwal2024} consider arbitrary interference within neighborhoods and develop algorithms based on discrete Fourier analysis, but their methods scale to at most $n \approx 12$ units. \citet{viviano} studies optimal policy targeting under network interference in an offline setting where the network of the target population may be unobserved. \citet{viviano_rudder2020} propose an adaptive experimental 
design for policy estimation under unknown interference using cluster-level perturbations. Our work differs in that we operate in a fully sequential setting, learn the interference network explicitly, and optimize individual-level treatment allocations without requiring cluster structure assumptions. \citet{gleich2026scalable} show that adopting an assumption on the 
additivity of direct and indirect effects
enables scalable Thompson sampling on networks with hundreds of nodes, but assume the interference network is known.

\section{Model}\label{sec:model}

\paragraph{Setup.} 
We consider sequential allocation of binary treatments to a fixed population of $n$ units over discrete time periods $t \in \{1,\ldots,T\}$. The units are connected by an \textbf{unknown} undirected network represented by the adjacency matrix $\mbf{A} \in \{0,1\}^{n\times n}$. 

Adopting the potential outcomes framework \citep{rubin1978}, we let $Y_i(\mathbf{Z}, \mathbf{A})$ 
denote the potential outcome for unit $i$ under treatment vector 
$\mathbf{Z} \in \{0,1\}^n$ and interference network $\mathbf{A}$. At each time period $t$, the experimenter chooses a treatment assignment 
$\mathbf{Z}_t \in \{0,1\}^n$ subject to a budget constraint 
$\|\mathbf{Z}_t\|_1 \leq B$, then observes noisy outcomes 
$Y_{t,i} = Y_i(\mathbf{Z}_t, \mathbf{A}) + \varepsilon_{t,i}$ for all $n$ 
units. The experimenter's goal is to maximize cumulative outcomes 
$\sum_{t=1}^T \sum_{i=1}^n Y_{t,i}$.

\paragraph{Outcome model.}
We assume that a unit's outcome depends only on its own treatment status 
and the treatment status of its neighbors as defined by $\mathbf{A}$. 
This is the Neighborhood Interference Assumption \citep{sussman2017} and 
is ubiquitous in the causal inference literature on network experiments 
\citep{toulis13, EcklesKarrerUgander, belloni} as well as in existing 
work on adaptive experimentation under network interference 
\citep{agarwal2024, gleich2026scalable}. Denote the set of unit $i$'s 
neighbors by $\mathcal{N}_i = \{j : A_{ij} = 1\}$.

\begin{assumption}
The outcome function of unit $i$ satisfies the Neighborhood Interference 
Assumption if for all treatment assignments $\mathbf{Z}_t, \mathbf{Z}_t'$ 
that agree on units $\mathcal{N}_i \cup \{i\}$, 
$Y_i(\mathbf{Z}_t, \mathbf{A}) = Y_i(\mathbf{Z}_t', \mathbf{A})$.
\end{assumption}

Under NIA, the expected potential outcome takes a parametric form that we denote $r_i(\mbf{Z}_t;\mbf{A}, \boldsymbol{\theta}) = \mathbb{E}[Y_i(\mbf{Z}_t, \mbf{A})].$ Crucially, the vector of unit-level outcome 
functions becomes linear in $\boldsymbol{\theta}$:
\[\mathbf{r}(\mathbf{Z}_t; \mathbf{A}, \boldsymbol{\theta}) = 
\mathbf{H}(\mathbf{Z}_t; \mathbf{A})\boldsymbol{\theta}.\]
Linearity is not an assumption; the form follows directly from NIA \citep{sussman2017}.

Under NIA, the dimension of $\boldsymbol{\theta}$ depends on the chosen parameterization of the reward function. Our algorithmic framework accommodates any specification that is linear in $\boldsymbol{\theta}$; that is, any specification that assumes NIA. This matches the generality of \citet{agarwal2024}. While this allows for a rich class of reward functions, generality comes at significant cost: the reward function of a node with degree $d$ has on the order of $2^{d+1}$ unknown parameters. Throughout, $\boldsymbol{\theta} \in \mathbb{R}^D$ is fixed-dimensional under a chosen reward parameterization; the graph $\mbf{A}$ enters the model only through the design matrix $\mathbf{H}(\mathbf{Z}_t; \mathbf{A})$, not through the dimension or indexing of $\boldsymbol{\theta}$. Learning quickly becomes infeasible for even moderate $n$, as seen in \citet{agarwal2024}; their method fails to scale beyond $n=12$. We provide experiments under several reward  specifications in Section~\ref{sec:experiments}, ranging from pure NIA as in \citet{agarwal2024} to the linear-in-means model of \citet{manski_ident}.

\paragraph{Maximization and Regret.}
The experimenter seeks to maximize cumulative node-level outcomes via the sequence of optimal treatment assignment vectors
\[\mbf{Z}_t^* = \arg \max_{\mbf{Z}:\|\mbf{Z}\|_1\leq B} \sum_{i=1}^nr_i(\mbf{Z};\mbf{A},\boldsymbol{\theta}):= \arg \max_{\mbf{Z}:\|\mbf{Z}\|_1\leq B} f_{\mbf{A}, \boldsymbol{\theta}}(\mbf{Z}), \]
where $\sum_{i=1}^nr_i(\mbf{Z};\mbf{A},\boldsymbol{\theta}) = f_{\mbf{A}, \boldsymbol{\theta}}(\mbf{Z}) $. We measure performance using regret, the expected difference between rewards under optimal decision making and the experimenter's policy:
\[\text{Reg}_T = \sum_{t=1}^T \sum_{i=1}^n \mathbb{E}\left[r_i(\mbf{Z}^*_t;\mbf{A}, \boldsymbol{\theta}) - r_i(\mbf{Z}_t;\mbf{A}, \boldsymbol{\theta})\right].\]
Expectations are taken both with respect to the noise $\boldsymbol{\epsilon}$ and any stochasticity in the experimenter's policy. For our Thompson Sampling-style algorithms, we consider Bayesian regret, which additionally integrates over the prior distributions on $\boldsymbol{\theta}$ and $\mbf{A}$.

Because both $\mbf{A}$ and $\boldsymbol{\theta}$ are unknown, the experimenter faces a joint learning problem in addition to the classic exploration-exploitation tradeoff: learning $\boldsymbol{\theta}$ requires knowledge of $\mbf{A}$ and vice-versa. 

\section{Thompson Sampling With a Joint Gibbs Sampler}
\label{sec:gibbs_ts}

In prior work on Thompson sampling under network interference \citep{gleich2026scalable}, the adjacency matrix $\mbf{A}$ is assumed to be observed at each round, so that the experimenter maintains a posterior only over $\boldsymbol{\theta}$. In many applications this assumption is unrealistic: the experimenter may suspect that social or behavioral ties produce interference without having a reliable roster of who interacts with whom. In this section we extend the Thompson sampling framework to the setting in which $\mbf{A}$ is unknown and must be learned jointly with $\boldsymbol{\theta}$ from the reward history.

\subsection{A Joint Prior and Gibbs Construction}\label{sec:gibbs_sampler}

To place the problem in a Bayesian framework, we specify a prior over $(\boldsymbol{\theta}, \mbf{A})$. Motivated by normal--normal conjugacy, we use a Gaussian prior $\boldsymbol{\theta} \sim \mathcal{N}(\boldsymbol{\mu}_0, \boldsymbol{\Sigma}_0)$, which is standard in Thompson sampling for linear bandits and is effectively equivalent to ridge regularization in the posterior mean. For the adjacency matrix we adopt an independent Bernoulli prior,
\[
A_{ij} \stackrel{\text{iid}}{\sim} \text{Bern}(\rho), \qquad 1 \leq i < j \leq n, \qquad A_{ij} = A_{ji}, \qquad A_{ii} = 0,
\]
where $\rho \in (0,1)$ is an edge-density hyperparameter. Denoting the history at the start of round $t$ by $\mathcal{D}_{t-1} = \{(\mbf{Z}_s, \mbf{r}_s)\}_{s=1}^{t-1}$, the joint posterior
\[
\pi(\boldsymbol{\theta}, \mbf{A} \mid \mathcal{D}_{t-1}) \propto \pi_0(\boldsymbol{\theta})\, \pi_0(\mbf{A}) \prod_{s=1}^{t-1} \mathcal{N}\!\left(\mbf{r}_s ; \mbf{H}(\mbf{Z}_s;\mbf{A})\boldsymbol{\theta}, \sigma^2 \mbf{I}_n\right)
\]
has no closed form because $\mbf{A}$ enters the likelihood nonlinearly through the neighbor counts $d^1_{s,i}(\mbf{A})$.  We therefore draw approximate joint samples via Gibbs sweeps that alternate between two full conditionals. Note that we use $\mbf{r}_s, r_{s,i}$ to denote the noisy rewards, abusing notation from Section~\ref{sec:model}.

\paragraph{(i) Parameter update: $\boldsymbol{\theta} \mid \mbf{A}, \mathcal{D}_{t-1}$.}
With $\mbf{A}$ fixed, the model reduces to a standard conjugate Gaussian regression. Stacking the rows of $\mbf{H}(\mbf{Z}_s;\mbf{A})$ for $s = 1, \ldots, t-1$ into a single design matrix $\mbf{H}_{t-1}(\mbf{A}) \in \mathbb{R}^{n(t-1)\times D}$, and letting $\mbf{r}_{1:t-1}$ denote the corresponding stacked reward vector, the full conditional is Gaussian:
\begin{align}
\boldsymbol{\Sigma}(\mbf{A}) &= \left(\mbf{H}_{t-1}(\mbf{A})^\top \mbf{H}_{t-1}(\mbf{A})/\sigma^2 + \boldsymbol{\Sigma}_0^{-1}\right)^{-1}, \label{eq:theta_cov} \\
\boldsymbol{\mu}(\mbf{A}) &= \boldsymbol{\Sigma}(\mbf{A})\left(\mbf{H}_{t-1}(\mbf{A})^\top \mbf{r}_{1:t-1}/\sigma^2 + \boldsymbol{\Sigma}_0^{-1}\boldsymbol{\mu}_0\right). \label{eq:theta_mean}
\end{align}

\paragraph{(ii) Edge update: $\mbf{A} \mid \boldsymbol{\theta}, \mathcal{D}_{t-1}$.}
With $\boldsymbol{\theta}$ fixed, we sweep over all edges $(i,j)$ with $i < j$, updating each edge from its Gibbs full conditional given all other edges and the data. Under NIA, the reward of node $\ell$ depends on $A_{ij}$ only if $\ell \in \{i,j\}$, so the full conditional for edge $(i,j)$ involves only the outcome streams at its two endpoints:
\begin{align}\label{eq:edge_gibbs}
&\operatorname{logit}\, P(A_{ij}=1 \mid \mbf{A}_{-(ij)}, \boldsymbol{\theta}, \mathcal{D}_{t-1})
= \log\frac{\rho}{1-\rho} \nonumber\\
&\quad + \frac{1}{2\sigma^2}\sum_{s=1}^{t-1} \sum_{\ell \in \{i,j\}} \left[
\left(r_{s,\ell} - \eta^{(0)}_{s,\ell}\right)^2 - \left(r_{s,\ell} - \eta^{(1)}_{s,\ell}\right)^2
\right],
\end{align}
where, for $b \in \{0,1\}$ and $\ell \in \{i,j\}$,
$\eta^{(b)}_{s,\ell} = h_\ell\!\left(\mbf{Z}_s; \mbf{A}^{(b)}_{\ell,\cdot}\right)^\top \boldsymbol{\theta}
$
is the predicted reward for node $\ell$ at round $s$ under the hypothesis $A_{ij} = b$, with all other edges held at their current values. Here $h_\ell(\cdot;\cdot)$ denotes the row of the design matrix $\mbf{H}$ corresponding to node $\ell$, determined by the chosen NIA-compatible parameterization, and $\mbf{A}^{(b)}_{\ell,\cdot}$ denotes row $\ell$ of the current adjacency matrix with the $(i,j)$ entry set to $b$. We sample $A_{ij} \sim \text{Bern}(p)$ where $p$ is the probability implied by Equation~\eqref{eq:edge_gibbs}, and set $A_{ji} = A_{ij}$.

Running $K$ Gibbs sweeps (alternating steps (i) and (ii)) at round $t$ yields an approximate joint sample $(\boldsymbol{\theta}^{(t)}, \mbf{A}^{(t)})$ from the posterior, which is then used to select $\mbf{Z}_t$ by maximizing the sampled expected total reward. To preserve posterior state across rounds, we warm-start the chain at each round with the final sample from the previous round. The full procedure is given as Algorithm~\ref{alg:gibbs_ts}.

\begin{algorithm}
\caption{Thompson Sampling with Edge-Wise Gibbs}
\begin{algorithmic}[1]
\State \textbf{Input:} Prior mean $\boldsymbol{\mu}_0$, prior covariance $\boldsymbol{\Sigma}_0$, noise variance $\sigma^2$, edge prior $\rho$, Gibbs iterations $K$
\State Initialize $\mbf{A}^{(0)}$ by drawing $A^{(0)}_{ij} \stackrel{\text{iid}}{\sim} \text{Bern}(\rho)$ for $i<j$ and symmetrizing
\For{$t = 1$ to $T$}
    \For{$k = 1$ to $K$}
        \State Build design matrix $\mbf{H}_{t-1}\!\left(\mbf{A}^{(k-1)}\right)$ from $\mathcal{D}_{t-1}$
        \State Sample $\boldsymbol{\theta}^{(k)} \sim \mathcal{N}\!\left(\boldsymbol{\mu}\!\left(\mbf{A}^{(k-1)}\right),\, \boldsymbol{\Sigma}\!\left(\mbf{A}^{(k-1)}\right)\right)$ via Eqs.~\eqref{eq:theta_cov}--\eqref{eq:theta_mean}
        \For{each pair $(i,j)$ with $i<j$}
            \State Compute $P(A_{ij}=1 \mid \mbf{A}_{-(ij)}, \boldsymbol{\theta}^{(k)}, \mathcal{D}_{t-1})$ via Eq.~\eqref{eq:edge_gibbs}
            \State Sample $A^{(k)}_{ij} \sim \text{Bern}(P)$ and set $A^{(k)}_{ji} \leftarrow A^{(k)}_{ij}$
        \EndFor
    \EndFor
    \State Set $(\boldsymbol{\theta}^{(t)}, \mbf{A}^{(t)}) \leftarrow (\boldsymbol{\theta}^{(K)}, \mbf{A}^{(K)})$
    \State Choose treatment vector:
    $\mbf{Z}_t = \arg\max_{\mbf{Z}:\|\mbf{Z}\|_1 \leq B} \mbf{1}_n^\top \!\left[\mbf{H}\!\left(\mbf{Z};\mbf{A}^{(t)}\right) \boldsymbol{\theta}^{(t)}\right]$
    \State Observe rewards $\mbf{r}_t$ and append $(\mbf{Z}_t, \mbf{r}_t)$ to $\mathcal{D}_t$
    \State Warm-start next round: $\mbf{A}^{(0)} \leftarrow \mbf{A}^{(t)}$
\EndFor
\end{algorithmic}
\label{alg:gibbs_ts}
\end{algorithm}

\paragraph{Computational complexity.}
The parameter update (Line 6) has cost $O(ntD^2 + D^3)$ per Gibbs sweep, identical to the known-graph algorithm of \citet{gleich2026scalable}. The edge update (Lines 7--9) sweeps over $\binom{n}{2}$ edges, and for each edge evaluates Equation~\eqref{eq:edge_gibbs} by scanning the $t-1$ historical rounds, giving $O(n^2 t)$ per sweep. Because the full conditional for each edge depends only on the reward streams at its two endpoints, no enumeration over edge configurations is required. The overall per-round complexity is $O(K(n^2 t + ntD^2 + D^3))$.

\paragraph{Practical considerations.}
We set $K=10$ Gibbs sweeps per round; Appendix~\ref{app:ablation} provides diagnostics consistent with adequate mixing. The Line~13 optimization is an integer linear program under general NIA, solved via Gurobi \citep{gurobi}; under linear-in-means it reduces to top-$B$ selection, eliminating the ILP solver.

\section{Theoretical Guarantees}
\label{sec:theory}

We present theoretical results under varying assumptions on the reward function. Under structural assumptions common within the causal inference literature, including additive spillovers and the linear-in-means model, total rewards depend on the network structure only through an $n$-dimensional vector of latent scores. Exploiting this, we prove a Bayesian regret bound for posterior sampling that is independent of graph topology. For general NIA rewards, we analyze an explore-then-commit strategy that decouples network and parameter learning, identifying $\mbf{A}$ through single-user treatments before running a known-graph algorithm on the recovered $\hat{\mbf{A}}$. Complementing both upper bounds, we prove an information-theoretic $\Omega(n\log T)$ lower bound on the cost of network discovery that holds for any consistent policy under general NIA rewards.

\subsection{Modular Collapse Under Additive Spillovers}
Consider reward functions of the form $r_i(\mbf{Z};\mbf{A},\boldsymbol{\theta})=\alpha_i+\tau_i Z_i+\sum_{j=1}^n \eta_{ij}Z_j$

where $\tau_i$ is the direct effect of treating unit $i$ and $\eta_{ij}$ is the indirect effect on unit $i$ from treating unit $j$ and $\eta_{ij}=0 \text{ if } A_{ij}=0.$ This family is standard in the causal inference literature and includes the linear-in-means model as a special case. Summing over units, the total reward becomes
\[\sum_{i=1}^n r_i(\mathbf{Z}; \mathbf{A}, \boldsymbol{\theta}) = c_{\mathbf{A}, \boldsymbol{\theta}} + \mathbf{Z}^\top \mathbf{s}_{\mathbf{A}, \boldsymbol{\theta}}, \qquad c_{\mathbf{A}, \boldsymbol{\theta}} = \sum_{i=1}^n \alpha_i, \quad s_j = \tau_j + \sum_{i=1}^n \eta_{ij}.\]
Total reward is \emph{linear} in $\mbf{Z}$ with coefficients determined by the score vector $\mbf{s}_{\mbf{A}, \boldsymbol{\theta}},$ reducing the problem to a linear bandit on $\mbf{Z}$ with latent coefficients determined by $\mbf{A}$ and $\boldsymbol{\theta}$.

\begin{assumption}[Unknown-graph modular collapse]
\label{ass:modular_collapse}
There exists a scalar $c_{\mbf{A},\boldsymbol{\theta}}$ and a score vector
$\mbf{s}_{\mbf{A},\boldsymbol{\theta}}\in\bb{R}^n$ such that, for every
$\mbf{Z}\in\cc{Z}_B$, $f_{\mbf{A},\boldsymbol{\theta}}(\mbf{Z}) = c_{\mbf{A},\boldsymbol{\theta}} + \mbf{Z}^{\top}\mbf{s}_{\mbf{A},\boldsymbol{\theta}}.$

\end{assumption}
The score vector $ \mbf{s}_{\mbf{A},\boldsymbol{\theta}}$ remains latent: it is a function of the unknown pair $(\mbf{A},\boldsymbol{\theta})$. The assumption asserts only that the optimization problem is linear in $\mbf{Z}$. We prove that the assumption holds for additive linear spillovers under NIA in Section~\ref{app:modular_corollaries}, including the linear-in-means model.

\begin{theorem}[Modular-collapse posterior sampling]\label{thm:modular_collapse}
Suppose Assumption~\ref{ass:modular_collapse} holds and the one-step
regret is bounded by $R$. Then exact posterior Thompson sampling over
$(\mbf{A},\boldsymbol{\theta})$ satisfies
\[
\operatorname{BayesReg}_T
  \le
  R
  \sqrt{
    \frac{nT\,\mathcal{H}(\mbf{Z}^*)}{2}
  }.
\]
Consequently,
\[
\operatorname{BayesReg}_T
  \le
  R
  \sqrt{
    \frac{nT\,B\log(en/B)}{2}
  }.
\]
\end{theorem}
The dimension factor is $n$, the length of the latent score vector, rather than the ambient dimension of $\theta$ as in \citet{gleich2026scalable}. The proof, given in Appendix~\ref{app:modular_collapse_proof}, applies the linear posterior-sampling information-ratio bound of \citet{russo2014learning} to the latent score vector. 

Theorem~\ref{thm:modular_collapse} bounds the Bayesian regret of exact posterior sampling. The Gibbs sampler of Section~\ref{sec:gibbs_sampler} produces approximate samples; regret analysis under approximation requires either a mixing-time bound for the edge-wise Gibbs sampler or a divergence based argument in the style of \citet{myphan}. \citet{mazumdar20a} establish that Langevin-based approximate Thompson sampling matches the regret rate of exact TS in linear bandits under standard mixing conditions. Whether analogous guarantees hold for our edge-wise Gibbs sampler is open. Empirically, the results of Section~\ref{sec:experiments} are consistent with the rate of Theorem~\ref{thm:modular_collapse} across all additive-spillover specifications, and Appendix~\ref{app:ablation} reports a sweep-count ablation with diagnostics consistent with adequate mixing.

\subsection{Explore-then-Commit with Thompson Sampling}

The algorithm operates in two phases. In Phase~1 (rounds $1$ to $T_0 = nm$), the experimenter treats each node $j \in \{1, \ldots, n\}$ in isolation for $m$ rounds, setting $\mbf{Z}_t = \textbf{e}_j$. When node $j$ is treated alone, each node $i \neq j$ has treated-neighbor count $d^1_{t,i} = A_{ij}$, producing reward $r_{t,i} = \gamma_1 A_{ij} + \varepsilon_{t,i}$. For each pair $(i,j)$, the experimenter computes the sample mean $\bar{r}_{ij}$ over $m$ observations and thresholds: $\hat{A}_{ij} = \mbf{1}\{\bar{r}_{ij} > \delta_\gamma / 2\}$. Here $\delta_\gamma$ is the minimum detectable spillover effect; this plays the same role as the minimum detectable effect size in classical power analysis and is assumed known to the experimenter. In Phase~2 (rounds $T_0 + 1$ to $T$), the experimenter runs the Thompson sampling algorithm of \citet{gleich2026scalable} using $\hat{\mbf{A}}$ as the known graph.

\begin{assumption}
\label{assumption:noise}
The noise terms $\varepsilon_{t,i}$ are independent $N(0, \sigma^2)$ random variables for all $t$ and $i$.
\end{assumption}

\begin{theorem}[ETC Upper Bound]
\label{thm:etc_upper}
Suppose Assumptions~1 and~\ref{assumption:noise} hold and that the spillover effect satisfies $\gamma_1 \geq \delta_\gamma >0$ for a known $\delta_\gamma$. Set $m = \lceil 8\sigma^2 \log(n^2 T) / \delta_\gamma^2 \rceil$. Then the explore-then-commit algorithm with Thompson sampling satisfies
\[
\bb{E}\!\left[\text{Reg}_T\right] \leq O\!\left(\frac{n^2 \sigma^2 (|\mu| + \gamma_{\max}) \log(nT)}{\delta_\gamma^2} + D\sqrt{nT\log(nT)}\right)
\]
where $\gamma_{\text{max}}$ is the maximum spillover effect.
\end{theorem}

The first term reflects the cost of graph discovery: it scales as $n^2$ (one test per potential edge) and does not grow with $T$. The second term reflects the cost of learning $\boldsymbol{\theta}$, matching the known-graph rate of \citet{gleich2026scalable}. The algorithm requires knowledge of both $T$ and $\delta_\gamma$.

\subsection{Lower Bound}

We now establish that any algorithm must pay a cost for network discovery that scales linearly in $n$. This bound applies to any consistent policy, including the joint Gibbs sampler.

\begin{theorem}[Lower Bound]
\label{thm:lower}
Consider the setting of Theorem~\ref{thm:etc_upper} with $B = 1$, $\mu = 0$, and $\gamma_1 > 0$ known. For $n = 2p$ nodes arranged in $p$ disjoint pairs, any consistent policy $\pi$ satisfies
\[
\sup_{\mbf{A}}\; \bb{E}_{\mbf{A}}^{\pi}\!\left[\text{Reg}_T\right] \;\geq\; \frac{2(p-1)\sigma^2}{\gamma_1}\log T.
\]
\end{theorem}

The lower bound establishes that the cost of graph discovery scales at least as $\Omega(n \sigma^2 \log T / \gamma_1)$. Comparing with the upper bound, the graph-learning terms differ in their dependence on $n$ ($n$ versus $n^2$) and in the parametric dependence ($1/\gamma_1$ versus $\gamma_{\max}/\delta_\gamma^2$). Closing this gap is an interesting direction for future work; we note that the $n^2$ factor in the upper bound arises from testing all $\binom{n}{2}$ potential edges, and could potentially be tightened under sparsity assumptions on $\mbf{A}$.

\section{Experiments}\label{sec:experiments}

\subsection{Head-to-Head Comparison at $n = 8$}
\label{sec:exp_head_to_head}

We compare four algorithms at $n = 8$, where \citet{agarwal2024} remains computationally feasible: Gibbs-TS, ETC-TS, Gibbs-TS fit with a misspecified additive reward (omitting the $\xi$ block; 36 parameters vs.\ 204 for the well-specified fit), and \citet{agarwal2024} using their published codebase. Data is generated under a full NIA reward function with pairwise neighbor interactions:
\[
r_i(\mbf{Z}; \mbf{A}) = \mu_i Z_i + \sum_{j \in \mathcal{N}_i} \gamma_{ij} Z_j + \sum_{\substack{j, k \in \mathcal{N}_i \\ j < k}} \xi_{ijk} Z_j Z_k + \varepsilon_i.
\]
We test two regimes: small $\xi$ ($\xi_{ijk} \sim \text{Uniform}[-0.4, 0.4]$, contributing ${\sim}3.5\%$ of expected reward) and large $\xi$ ($\xi_{ijk} \sim \text{Uniform}[-3, 3]$, ${\sim}25\text{--}30\%$). Reward parameters are sampled $\mu\sim \text{Uniform}[0.5, 1.5]$ and $\gamma_{ij} \sim \text{Uniform}[0.3,1.0]$. Networks are Erdős-Rényi graphs with $p=0.3$. We set $\sigma^2 = 0.5$ and $B=3$. Figure~\ref{fig:nia_quadratic} presents results of $30$ independent trials for each setting.

\begin{figure}[htbp]
\centering
\includegraphics[width=0.98\textwidth]{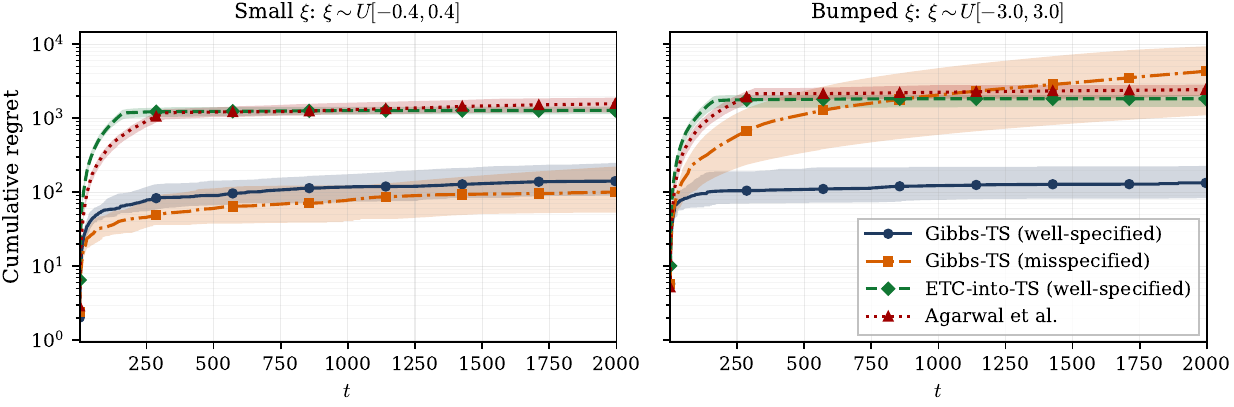}
\caption{Cumulative regret at $n=8$ on the pairwise-interaction reward function. Left: small $\xi$. Right: large $\xi$. Median regret at $T=2000$ for small/large $\xi$: well-specified Gibbs-TS 141/134, additive Gibbs-TS 101/4358, ETC-TS 1276/1833, \citet{agarwal2024} 1571/2435.}
\label{fig:nia_quadratic}
\end{figure}

Well-specified Gibbs-TS is stable across both regimes and beats baselines by 9--18$\times$. The additive fit wins under small $\xi$ but degrades sharply under large $\xi$. Theorem~\ref{thm:modular_collapse} does not formally cover this reward function (pairwise interactions violate modular collapse), yet Gibbs-TS achieves sublinear regret in both regimes, suggesting the algorithm's empirical performance extends beyond the theorem's formal scope. Appendices~\ref{app:nia_specs}--\ref{app:rho_sensitivity} verify the ordering across additional specifications, examine misspecification robustness, and report sensitivity to the edge prior $\rho$.

\subsection{Real-Network Experiments}
\label{sec:exp_real}

We further evaluate Gibbs-TS on two networks drawn from published applied studies with simulated outcomes. The first is the borrowing-money household network for village~10 from \citet{banerjee_diffusion2013}, with $n = 63$ households after dropping isolates, $114$ undirected edges, and density $0.058$. The second is the symmetrized friendship roster for school SCHID~3 from the anti-conflict experiment of \citet{paluck_shepherd_aranow2016}, used as a benchmark by \citet{aronow_samii}, with $n = 121$ students after dropping isolates, $251$ undirected edges, and density $0.035$.

We adopt the linear-in-means reward function:  $r_i = \mu_i Z_i + \beta_i \bar{Z}_{\mathcal{N}_i} + \varepsilon_i$, where $\bar{Z}_{\mathcal{N}_i}$ denotes the fraction of $i$'s neighbors that are treated. We set $\sigma = 1.0$, $T = 10{,}000$, and $B = \lceil n/5 \rceil$. Reward parameters are sampled $\mu_i, \beta_i \sim \text{Uniform}[0, 1]$. No prior adaptive method that learns the interference network scales to these sizes, so we benchmark Gibbs-TS against the Known-$\mbf{A}$ TS oracle of \citet{gleich2026scalable}, which runs Thompson sampling on the true network. We also include a no-interference TS baseline that assumes independence ($\mbf{A} = \boldsymbol{0}_{n\times n}).$

As seen in Figure~\ref{fig:real_data}, regret grows sublinearly in $t$ on both networks, indicating that joint learning continues to be effective at applied network sizes. The remaining gap to the Known-$\mbf{A}$ oracle, roughly $4.5\times$ on the village and $3.8\times$ on the school, is attributable to network learning rather than parameter learning, since both methods share the same Thompson sampler conditional on the network. Appendix~\ref{app:linmeans_scaling} reports a scaling study under the linear-in-means specification at $n \in \{50, 100, 250, 500, 1000\}$.

\subsection{Downstream causal estimation}
\label{sec:exp_downstream}

The recovered $\hat{\mathbf{A}}$ supports estimation of treatment effects through standard causal inference machinery. We consider three target estimands under the true $(\mbf{A}, \boldsymbol{\theta})$: the direct effect $\tau_D$, the indirect effect $\tau_{I}(1)$, and the total treatment effect $\tau_{\text{TTE}}.$ We compare three estimators: (i) the closed-form Gibbs posterior mean of $\boldsymbol{\theta}$ from the adaptive phase under $\hat{\mathbf{A}}$; (ii) OLS under $\hat{\mathbf{A}}$ fit to a separate randomized inference phase ($T_{\text{eval}} = T$ assignments drawn $Z_{t,i} \overset{\text{iid}}{\sim} \text{Bernoulli}(B/n)$, outcomes generated under the true $(\boldsymbol{\theta}^*, \mathbf{A})$); and (iii) the same randomized OLS fit under the true $\mathbf{A}$ as an oracle. The contrast between (ii) and (iii) isolates the downstream cost of graph misspecification, while the contrast between (i) and (ii) compares the adaptive posterior mean to a randomized-design benchmark. We set $\hat{\mbf{A}}$ to be the marginal posterior mean of each edge at round $T$, thresholded at $0.5$.

\begin{table}[h]
\centering
\small
\caption{RMSE of treatment-effect estimates under $\hat{\mathbf{A}}$ and the
true $\mathbf{A}$. Left: $n=20$ SBM (two groups, $p_{\text{within}} = 0.25$, $p_{\text{between}}= 0.05$), count-based NIA reward, 15 reps,
median F1 $= 0.98$. Right: village~10, linear-in-means reward, 10 reps,
median F1 $= 0.74$. Estimands are defined under the true
$(\mathbf{A}, \boldsymbol{\theta}^*)$.}
\label{tab:downstream}
\begin{minipage}{0.48\textwidth}
\centering
\begin{tabular}{@{}lccc@{}}
\toprule
Estimand & Gibbs ($\hat{\mathbf{A}}$) & OLS ($\hat{\mathbf{A}}$) & OLS ($\mathbf{A}$) \\
\midrule
$\tau_D = 0.97$            & 0.013 & 0.016 & 0.012 \\
$\tau_I(1) = 1.01$         & 0.011 & 0.035 & 0.009 \\
$\tau_{\text{TTE}} = 2.76$ & 0.123 & 0.094 & 0.017 \\
\bottomrule
\end{tabular}
\subcaption*{(a) $n=20$ SBM, count-based NIA}
\end{minipage}
\hfill
\begin{minipage}{0.48\textwidth}
\centering
\begin{tabular}{@{}lccc@{}}
\toprule
Estimand & Gibbs ($\hat{\mathbf{A}}$) & OLS ($\hat{\mathbf{A}}$) & OLS ($\mathbf{A}$) \\
\midrule
$\tau_D = 0.51$            & 0.039 & 0.037 & 0.002 \\
$\tau_I(1) = 0.22$         & 0.016 & 0.021 & 0.002 \\
$\tau_{\text{TTE}} = 1.00$ & 0.129 & 0.076 & 0.004 \\
\bottomrule
\end{tabular}
\subcaption*{(b) village~10, linear-in-means}
\end{minipage}
\end{table}
Table~\ref{tab:downstream} reports results in two settings. On the $n=20$ SBM, randomized OLS under $\hat{\mathbf{A}}$ closely tracks the oracle on $\tau_D$ and $\tau_{I}(1)$, with substantial degradation on $\tau_{\text{TTE}}$ that depends on the empirical degree distribution. On village~10, the cost of misspecification is larger across all three estimands but remains small relative to the true effect sizes. Two patterns are robust across settings: $\tau_D$ and $\tau_I(1)$ are estimable from $\hat{\mathbf{A}}$ at small cost when recovery is high, while $\tau_{\text{TTE}}$ depends on the full degree distribution and inherits more error from imperfect recovery.

\section{Limitations and Conclusion}
\label{sec:discussion}

\paragraph{Limitations.} Several questions remain open: closing the $n$ versus $n^2$ gap between the upper and lower bounds, extending modular collapse to richer spillover models, and a mixing-time analysis connecting the modular-collapse bound to the approximate Gibbs sampler. The framework relies on correct specification of the reward function; further generality would be useful for applications. The ILP solver becomes a bottleneck for large $n$: faster solvers would increase scalability. Further theoretical analysis conditional on standard sparsity assumptions on the graph could improve the theoretical bounds.

\paragraph{Conclusion.} We have presented an adaptive experimental design that jointly recovers the interference network and optimizes individual-level treatment policies. The paper makes both algorithmic and theoretical contributions. Algorithmically, an edge-wise Gibbs sampler whose full conditional factors across the endpoints of each edge enables joint posterior sampling, producing the recovered network $\hat{\mathbf{A}}$ as an explicit deliverable for downstream causal analysis (Section~\ref{sec:exp_downstream}). Theoretically, we prove a Bayesian regret bound of $\sqrt{nT \cdot B \log(en/B)}$ for exact posterior sampling under additive spillover models (Theorem~\ref{thm:modular_collapse}). For general Neighborhood Interference, an explore-then-commit upper bound and a complementary information-theoretic lower bound establish the cost of network discovery up to a factor of $n$. We see particular promise in settings where the interference network is itself a scientific deliverable: peer-effect estimation in education and labor economics, contagion networks in epidemiology, and social influence in online platforms.


\bibliography{papers}
\bibliographystyle{apalike}

\newpage 

\appendix
\section{Proof of Theorem~\ref{thm:modular_collapse}}\label{app:modular_collapse_proof}
We begin with a Lemma:

\begin{lemma}[Linear posterior-sampling information ratio]
\label{lem:route_b_linear_ir}
Consider posterior sampling for a linear objective
\[
f_{\mbf{s}}(\mbf{z})
  =
  \mbf{z}^{\top}\mbf{s},
\qquad
\mbf{s}\in\bb{R}^n,
\qquad
\mbf{z}\in\cc{Z}_B.
\]
If one-step regret is bounded by $R$, then the conditional information ratio
with scalar total-reward feedback satisfies
\[
\Gamma_t
  \le
  \frac{R^2 n}{2}.
\]
\end{lemma}

\begin{proof}
This is the standard linear posterior-sampling information-ratio bound of Russo and Van Roy specialized to an $n$-dimensional linear objective, as justified by Lemma 1 and Lemma 2 in \citet{gleich2026scalable} which extend the linear bandit theory to the case of an observed vector reward. The only modification needed is that the linear coefficient is the latent score vector rather than a directly observed parameter; since the IR bound depends only on the dimension of the linear coefficient, the same bound applies. The finite feasible set $\cc{Z}_B$ affects the entropy term in the final regret bound, but not the dimension factor in the information ratio.
\end{proof}
We now prove Theorem~\ref{thm:modular_collapse}:
\begin{proof}
Conditional on $\mathcal{D}_{t-1}$, exact posterior sampling over
$(\mbf{A},\boldsymbol{\theta})$ induces exact posterior sampling over the
latent score vector $\mbf{s}_{\mbf{A},\boldsymbol{\theta}}$. Under
Assumption~\ref{ass:modular_collapse}, the true optimal action
$\mbf{Z}^*$ and the sampled action $\mbf{Z}_t$ are both determined by the same
optimization map applied to the true and sampled score vectors. Hence the
posterior-sampling identity holds: conditional on $\mathcal{D}_{t-1}$,
$\mbf{Z}_t$ and $\mbf{Z}^*$ have the same distribution.

The IR bound of Lemma~\ref{lem:route_b_linear_ir} holds for any prior on $\mbf{s}$, since posterior sampling depends on the prior only through the posterior distribution of $\mbf{s}\mid\mathcal{D}_{t-1}$. The induced prior here, obtained as the pushforward of $\pi_0(\mbf{A},\boldsymbol{\theta})$ under the deterministic map $(\mbf{A},\boldsymbol{\theta})\mapsto\mbf{s}_{\mbf{A},\boldsymbol{\theta}}$,
is one such prior.

Let
\[
\Delta_t
  =
  \bb{E}_t[
    f_{\mbf{A},\boldsymbol{\theta}}(\mbf{Z}^*)
    -
    f_{\mbf{A},\boldsymbol{\theta}}(\mbf{Z}_t)
  ]
\]
and let $Y_t^{\mathrm{tot}}=\mbf{1}_n^\top\mbf{r}_t$ be the scalar total
reward. Lemma~\ref{lem:route_b_linear_ir} gives
\[
\Delta_t^2
  \le
  \frac{R^2n}{2}
  I_t(\mbf{Z}^*;(\mbf{Z}_t,Y_t^{\mathrm{tot}})).
\]
The actual observation is the full reward vector $\mbf{r}_t$, and
$Y_t^{\mathrm{tot}}$ is a measurable function of $\mbf{r}_t$. Therefore, by
data processing,
\[ 
I_t(\mbf{Z}^*;(\mbf{Z}_t,\mbf{r}_t))
  \ge
I_t(\mbf{Z}^*;(\mbf{Z}_t,Y_t^{\mathrm{tot}})).
\]

Since the IR is decreasing in the denominator $I_t$, the bound obtained in the scalar observation setting continues to hold when $I_t$ is replaced by the larger information from the node-level reward vector. 

The standard posterior-sampling information-ratio telescoping argument gives
\[
\operatorname{BayesReg}_T
  \le
  R
  \sqrt{
    \frac{nT\,\mathcal{H}(\mbf{Z}^*)}{2}
  }.
\]
Finally,
\[
H(\mbf{Z}^*)
  \le
  \log |\cc{Z}_B|
  \le
  B\log(en/B),
\]
which proves the second display.
\end{proof}

\subsection{Corollaries}\label{app:modular_corollaries}
We now show that Assumption~\ref{ass:modular_collapse} holds for a general class of additive reward functions. For added clarity, we show how common but more restrictive classes of reward functions factor into the linear formulation.

\begin{corollary}[Additive directed edge spillovers]
\label{cor:additive_edge_spillovers}
Suppose rewards take the form
\[
r_i(\mbf{z};\mbf{A},\boldsymbol{\theta})
  =
  \alpha_i+\tau_i z_i+\sum_{j=1}^n \eta_{ij}z_j,
\qquad
\eta_{ij}=0 \text{ if } A_{ij}=0.
\]
Then Assumption~\ref{ass:modular_collapse} holds with
\[
c_{\mbf{A},\boldsymbol{\theta}}
  =
  \sum_{i=1}^n \alpha_i,
\qquad
s_j
  =
  \tau_j+\sum_{i=1}^n\eta_{ij}.
\]
Therefore the regret bound of
Theorem~\ref{thm:modular_collapse} applies to ideal
posterior sampling over the unknown graph and reward parameters.
\end{corollary}

\begin{proof}
Summing the node-level rewards gives
\[
\sum_{i=1}^n r_i(\mbf{z};\mbf{A},\boldsymbol{\theta})
  =
  \sum_{i=1}^n \alpha_i
  +
  \sum_{i=1}^n \tau_i z_i
  +
  \sum_{i=1}^n\sum_{j=1}^n \eta_{ij}z_j
  =
  c_{\mbf{A},\boldsymbol{\theta}}
  +
  \sum_{j=1}^n z_j
  \left(\tau_j+\sum_{i=1}^n\eta_{ij}\right).
\]
This is exactly the modular-collapse representation.
\end{proof}

\begin{corollary}[Linear treated-neighbor exposure]
\label{cor:linear_treated_neighbor}
Suppose
\[
r_i(\mbf{z};\mbf{A},\boldsymbol{\theta})
  =
  \alpha_i+\tau_i z_i
  +
  \lambda_i d_i^1(\mbf{z};\mbf{A}),
\qquad
d_i^1(\mbf{z};\mbf{A})=\sum_{j=1}^n A_{ij}z_j .
\]
Then Assumption~\ref{ass:modular_collapse} holds with
\[
s_j
  =
  \tau_j+\sum_{i=1}^n \lambda_i A_{ij}.
\]
Therefore the regret bound of
Theorem~\ref{thm:modular_collapse} applies.
\end{corollary}

\begin{proof}
The spillover term can be rewritten as
\[
\lambda_i d_i^1(\mbf{z};\mbf{A})
  =
  \sum_{j=1}^n \lambda_i A_{ij}z_j.
\]
This is the additive edge-spillover model of
Corollary~\ref{cor:additive_edge_spillovers} with
$\eta_{ij}=\lambda_iA_{ij}$.
\end{proof}

\begin{corollary}[Homogeneous linear spillovers]
\label{cor:homogeneous_linear_spillovers}
Suppose
\[
r_i(\mbf{z};\mbf{A},\boldsymbol{\theta})
  =
  \alpha_i+\tau_i z_i
  +
  \lambda d_i^1(\mbf{z};\mbf{A}).
\]
Then Assumption~\ref{ass:modular_collapse} holds with
\[
s_j
  =
  \tau_j+\lambda \deg_{\mbf{A}}(j).
\]
Thus the unknown graph matters for treatment optimization only through the
degree-weighted latent scores.
\end{corollary}

\begin{proof}
This is Corollary~\ref{cor:linear_treated_neighbor} with
$\lambda_i=\lambda$ for every $i$. Since $\mbf{A}$ is undirected,
$\sum_i A_{ij}=\deg_{\mbf{A}}(j)$.
\end{proof}

\section{Proof of Theorem~\ref{thm:lower}}
\label{app:lower_proof}

Suppose the node-level rewards are of the form
\[r_i(\mbf{Z}_t) = \mu \cdot Z_{t,i} + \gamma_1 \mbf{1}\{d_{t,i}^1 = 1\}.\]
That is, there is a direct effect and a spillover effect from having exactly one treated neighbor. Our hard instance involves a graph where nodes have at most one neighbor, motivating this reward function.

\begin{proof}
For each $k \in \{1,\ldots, p\}$, define $\mbf{A}^{(k)}$ as the graph where only pair $k$ is connected. The only nonzero expected reward comes from treating either node in pair $k$. The agent's action at each step is the choice of which pair to treat.

Define $u_k$ and $v_k$ as the two nodes in pair $k$. If the agent treats $u_k$, she observes $r_{t,v_k} = \gamma_1 + \varepsilon_{t,v_k}$ and $r_{t,i} = \varepsilon_{t,i}$ for all $i \neq v_k$. While the agent observes $n$ rewards per round (unlike one in a standard $m$-armed bandit), $r_{t,v_k}$ is sufficient for the identity of the connected pair $k$: all other rewards are independent of the action choice and provide no information about which pair is connected. When the agent treats a pair $j \neq k$, all node-level rewards have mean zero.

This is an $m$-armed Gaussian bandit with mean $\gamma_1$ if the agent pulls arm $k$ and mean $0$ otherwise. By Theorem~16.2 of \citet{Lattimore_Szepesvári_2020}, any consistent policy satisfies $\liminf_{T \to \infty} \bb{E}[N_j(T)] / \log T \geq 2\sigma^2 / \gamma_1^2$ for each suboptimal arm $j \neq k$. Each pull of arm $j \neq k$ incurs regret $\gamma_1$. Summing over the $p-1$ suboptimal arms yields the result.
\end{proof}

\section{Proof of Theorem~\ref{thm:etc_upper}}
\label{app:etc_proof}

\begin{proof}
\textbf{Phase 1 (rounds $1$ to $T_0 = nm$).} The agent cycles through each node $j \in \{1, \ldots, n\}$, treating it alone for $m$ rounds by setting $\mbf{Z}_t = e_j$. When node $j$ is treated alone, each other node $i$ has treated-neighbor count $d_{t,i}^1 = A_{ij}$, so
\[
r_{t,i} = \gamma_1 A_{ij} + \varepsilon_{t,i}.
\]
For each pair $(i,j)$, the agent computes the sample mean $\bar{r}_{ij}$ over $m$ observations and thresholds: set $\hat{A}_{ij} = 1$ if $\bar{r}_{ij} > \delta_\gamma / 2$, else $\hat{A}_{ij} = 0$.

If $A_{ij} = 0$, then $\bar{r}_{ij}$ has mean $0$ and the error probability is
\[
P(\bar{r}_{ij} > \delta_\gamma / 2) \leq \exp\!\left(-\frac{m\delta_\gamma^2}{8\sigma^2}\right).
\]
If $A_{ij} = 1$, then $\bar{r}_{ij}$ has mean $\gamma_1 \geq \delta_\gamma$ and the error probability is
\[
P(\bar{r}_{ij} \leq \delta_\gamma / 2) = P(\bar{r}_{ij} - \gamma_1 \leq \delta_\gamma/2 - \gamma_1) \leq \exp\!\left(-\frac{m(\gamma_1 - \delta_\gamma/2)^2}{2\sigma^2}\right) \leq \exp\!\left(-\frac{m\delta_\gamma^2}{8\sigma^2}\right),
\]
where the last inequality uses $\gamma_1 \geq \delta_\gamma$, which implies $\gamma_1 - \delta_\gamma/2 \geq \delta_\gamma/2$. By our choice of $m$, we have $P(\hat{A}_{ij} \neq A_{ij}) \leq 1/(n^2 T)$. Taking a union bound over all $\binom{n}{2} \leq n^2/2$ pairs:
\[
P(\hat{\mbf{A}} \neq \mbf{A}) \leq \frac{1}{2T}.
\]

Each round of Phase~1 incurs regret at most $R_{\max} := n(|\mu| + \gamma_{\max})$. The total Phase~1 regret is at most
\[
T_0 \cdot R_{\max} = O\!\left(\frac{n^2 \sigma^2 (|\mu| + \gamma_{\max}) \log(nT)}{\delta_\gamma^2}\right).
\]

\textbf{Phase 2 (rounds $T_0 + 1$ to $T$).} Condition on the event $\{\hat{\mbf{A}} = \mbf{A}\}$, which holds with probability $\geq 1 - 1/(2T)$. On this event, the estimated graph equals the true graph and the Thompson sampling algorithm of \citet{gleich2026scalable} applies directly to the remaining $T - T_0 \leq T$ rounds, giving Phase~2 regret $O(D\sqrt{nT\log(nT)})$. On the complementary event $\{\hat{\mbf{A}} \neq \mbf{A}\}$, which occurs with probability $\leq 1/(2T)$, worst-case regret over $T$ rounds is at most $T \cdot R_{\max}$.

Taking expectations over the graph recovery event:
\[
\bb{E}[\text{Reg}_T] \leq \frac{n^2 \sigma^2 (|\mu| + \gamma_{\max}) \log(nT)}{\delta_\gamma^2} + O\!\left(D\sqrt{nT\log(nT)}\right) + T \cdot R_{\max} \cdot \frac{1}{2T}.
\]
The third term is $R_{\max}/2 = O(n(|\mu| + \gamma_{\max}))$, which is absorbed into the first term.
\end{proof}

\section{Additional Experiments}
\subsection{Further Comparisons to \citet{agarwal2024}}
\label{app:nia_specs}

We use two additional reward functions, each drawn from the full NIA parameterization, chosen to exercise different modes of interference.

\paragraph{Spec A: saturation.}
The direct effect of treating node $i$ is active only when none of $i$'s neighbors are also treated:
\[
r_i(\mbf{Z}; \mbf{A}) = \mu_i Z_i \cdot \mbf{1}\{d^1_i = 0\} + \sum_{j \in \mathcal{N}_i} \gamma_{ij} Z_j + \varepsilon_i.
\]
This captures settings in which a unit's own treatment is informative only when the unit is otherwise unexposed, such as a vaccine whose efficacy for the recipient is swamped by transmission reductions in highly vaccinated neighborhoods. Spillover effects are identity-specific via $\gamma_{ij}$, so the reward depends on which neighbors are treated rather than only on how many.

\paragraph{Spec B: treatment--spillover interaction.}
The direct and spillover effects interact through the count of treated neighbors:
\[
r_i(\mbf{Z}; \mbf{A}) = \mu_i Z_i + \sum_{j \in \mathcal{N}_i} \gamma_{ij} Z_j + \lambda_i Z_i \cdot d^1_i + \varepsilon_i.
\]
The interaction parameter $\lambda_i$ can take either sign, capturing settings in which treatment amplifies or dampens spillover exposure.

For both specifications we sample $\mu_i \sim \text{Uniform}[0.5, 1.5]$ and $\gamma_{ij} \sim \text{Uniform}[0.3, 1.0]$ independently per rep. For specification B, $\lambda_i \sim \text{Uniform}[-0.3, 0.3]$. We set $\sigma = 0.5$, $T = 2000$, and $B = 3$. The Thompson sampling algorithm uses a Gaussian prior on $\boldsymbol{\theta}$ with $\boldsymbol{\mu}_0 = \mbf{0}$, $\boldsymbol{\Sigma}_0 = 10 \mbf{I}$, and an edge prior $\rho = 0.3$. We use $K = 10$ Gibbs sweeps. The ETC variant uses $m = 20$ rounds of isolated treatment per node. The \citet{agarwal2024} implementation uses their default Lasso regularization selected by cross-validation.

Figure~\ref{fig:nia_specs} reports cumulative regret over time. On both specifications, Thompson sampling with the joint Gibbs sampler achieves substantially lower regret than either baseline. Median cumulative regret at $T = 2000$ is 64 under specification A and 93 under specification B, compared to 1113 and 1176 for ETC and 1342 and 1475 for \citet{agarwal2024}: an improvement of more than an order of magnitude over the nearest competitor. The ordering is consistent across reward functions, indicating that the algorithm's advantage does not depend on a particular mode of interference.

\begin{figure}[htbp]
\centering
\includegraphics[width=0.98\textwidth]{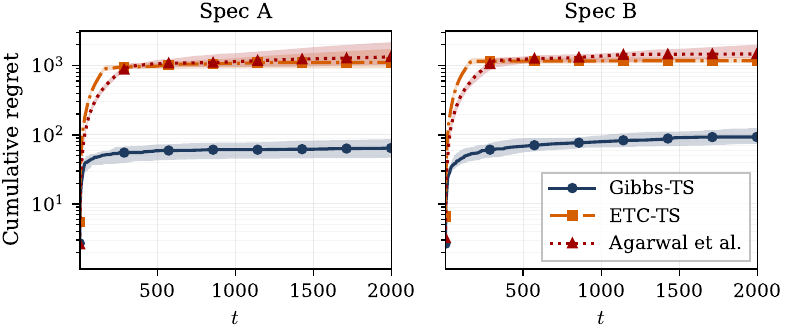}
\caption{Cumulative regret at $n = 8$ under two NIA reward functions. Left: saturation. Right: treatment--spillover interaction. Shaded bands denote 25th--75th percentile across 30 replications. Joint Gibbs Thompson sampling achieves more than an order-of-magnitude improvement over the nearest baseline on both reward functions.}
\label{fig:nia_specs}
\end{figure}

\subsection{Robustness to Reward Model Misspecification}
\label{app:misspec}

The head-to-head experiments of Section~\ref{sec:exp_head_to_head} grant Gibbs-TS and ETC-TS a fit model that matches the truth. To test sensitivity to reward model misspecification, we fit both methods to a count-based NIA specification, $r_i = \mu_i Z_i + \sum_k \gamma_{i,k} \mathbf{1}\{d^1_i = k\}$, which represents spillover only through the count $d^1_i$ and has no treatment-by-spillover cross term. Under specification A this drops the per-edge heterogeneity in $\gamma_{ij}$ and the saturation gate on the direct effect; under specification B it drops the per-edge heterogeneity and the $Z_i \cdot d^1_i$ interaction. \citet{agarwal2024} operates over the full $2^n$ exposure space and carries no such representation gap.

\begin{table}[h]
\centering
\small
\caption{Median cumulative regret at $T=2000$ under reward model misspecification. Gibbs-TS and ETC-TS fit a misspecified reward model; \citet{agarwal2024} is well-specified.}
\label{tab:misspec}
\begin{tabular}{@{}lcc@{}}
\toprule
 & Spec A & Spec B \\
\midrule
Gibbs-TS (misspec.) & 333.5 & 243.4 \\
ETC-TS (misspec.) & 1214.8 & 1799.3 \\
Agarwal et al. (2024) & 1342.3 & 1475.0 \\
\bottomrule
\end{tabular}
\end{table}

Table~\ref{tab:misspec} reports median cumulative regret at $T=2000$. Gibbs-TS retains a 4--6$\times$ advantage over \citet{agarwal2024} despite the representation gap, with IQRs of $[134.6, 1242.9]$ on Spec A and $[113.4, 514.2]$ on Spec B. The cost of misspecification is real, roughly $5\times$ relative to the well-specified Gibbs-TS of Section~\ref{sec:exp_head_to_head}.

\subsection{Sensitivity to the Edge Prior $\rho$}
\label{app:rho_sensitivity}

We rerun the small-$\xi$ head-to-head experiment of Section~\ref{sec:exp_head_to_head} at $n=8$ with the edge prior swept over $\rho \in \{0.05, 0.15, 0.30, 0.50, 0.70\}$, where the true edge density is $0.30$. We use $15$ seeds per value with matched seeds across $\rho$, so within-seed differences in final regret are attributable solely to the prior. Table~\ref{tab:rho_sensitivity} reports median final regret at $T=2{,}000$ along with the interquartile range. Performance is robust across the grid: median regret varies within a factor of $1.5$ even when $\rho$ is misspecified by $6\times$ below or $2.3\times$ above the truth, and the IQRs overlap across all five settings. Heavier right tails at $\rho \geq 0.50$ reflect occasional slow pruning of dense initial graphs; medians, robust to these tails, remain stable.

\begin{table}[h]
\centering
\small
\caption{Sensitivity of Gibbs-TS to the edge prior $\rho$ on the small-$\xi$ pairwise-interaction reward at $n=8$, $T=2{,}000$, 15 matched seeds. The true edge density is $0.30$. Median final regret is stable across an order of magnitude of misspecification.}
\label{tab:rho_sensitivity}
\begin{tabular}{@{}lccccc@{}}
\toprule
$\rho$            & $0.05$         & $0.15$         & $0.30$         & $0.50$         & $0.70$         \\
\midrule
Median regret     & $174.7$        & $120.6$        & $139.3$        & $114.0$        & $121.8$        \\
IQR               & $[104, 324]$   & $[79, 204]$    & $[86, 275]$    & $[89, 313]$    & $[91, 230]$    \\
\bottomrule
\end{tabular}
\end{table}

\subsection{Count-Based NIA at $n\in\{20,40\}$}
\label{sec:exp_sania_scaling}

To evaluate scaling beyond the reach of \citet{agarwal2024}, we run Gibbs-TS and ETC-TS at $n \in \{20, 40\}$ under a count-based NIA reward function with shared parameters:
\[
r_i(\mbf{Z}; \mbf{A}) = \mu Z_i + \sum_{k=1}^{d_{\max}} \gamma_k \mbf{1}\{d^1_i = k\} + \varepsilon_i.
\]
We sample $\mu \sim \text{N}(1, 0.2)$ and $\gamma_k \sim \text{N}(k, 0.5)$ independently per replication. At $n = 20$ we set $T = 2{,}000$ and $B = 6$; at $n = 40$ we set $T = 5{,}000$ and $B = 13$. 

\begin{table}[htbp]
\caption{Median cumulative regret (IQR) and network recovery accuracy under count-based NIA at $n=20$ ($T=2{,}000$, 30 reps) and $n=40$ ($T=5{,}000$, 10 reps).}
\label{tab:scaling_sania}
\centering
\small
\begin{tabular}{@{}llcc@{}}
\toprule
 & & $n=20$ & $n=40$ \\
\midrule
\multirow{3}{*}{Median regret (IQR)}
 & Known-A TS & 57 [40, 67] & 56 [50, 59] \\
 & Gibbs-TS & 351 [302, 425] & 1,441 [1,050, 3,205] \\
 & ETC-TS & 28,637 [23,951, 32,545] & 131,748 [115,658, 143,938] \\
\midrule
\multirow{2}{*}{Network acc.}
 & Gibbs-TS & 1.000 [0.995, 1.000] & 0.999 [0.996, 1.000] \\
 & ETC-TS & 0.995 [0.991, 1.000] & 0.999 [0.995, 1.000] \\
\bottomrule
\end{tabular}
\end{table}

Table~\ref{tab:scaling_sania} reports median cumulative regret and network recovery accuracy. Both methods achieve network recovery accuracy above $0.99$. Gibbs-TS lies within $6\times$ of the Known-$\mbf{A}$ oracle at $n=20$ and $26\times$ at $n=40$, while ETC-TS exceeds the oracle by roughly $500\times$ and $2{,}300\times$ respectively. Both methods recover the network with median accuracy above $0.99$ at both scales, isolating the residual regret gap to parameter learning rather than network discovery.

\subsection{Scaling to Large Networks}
\label{app:linmeans_scaling}

Under stronger assumptions on the reward function, the algorithm scales to networks of a thousand units. Adopting the linear-in-means parameterization standard in the peer effects literature \citep{manski_ident, BRAMOULLE200941}, node $i$'s reward takes the form
\[
r_i(\mbf{Z}; \mbf{A}) = \mu Z_i + \beta \bar{Z}_{\mathcal{N}_i} + \varepsilon_i,
\]
where $\bar{Z}_{\mathcal{N}_i}$ denotes the fraction of $i$'s neighbors that are treated. This reduces the parameter count to two per node and transforms the treatment selection step from an integer linear program to selection of the top-$B$ nodes by expected marginal reward, eliminating the ILP solver entirely.

We evaluate Gibbs-TS at $n \in \{50, 100, 250, 500, 1000\}$, with $T = 40{,}000$, $\sigma = 1.0$, and $B = \lceil n/5 \rceil$. Parameters are drawn as $\mu \sim \text{N}(2,1)$ and $ \beta \sim \text{N}(1,0.5)$.

Figure~\ref{fig:linmeans} shows cumulative regret over time at each network size. Regret grows sublinearly in $t$ at every scale, and the shape is consistent across $n$, indicating that learning remains effective as the network grows.

\begin{figure}[htbp]
\centering
\includegraphics[width=0.98\textwidth]{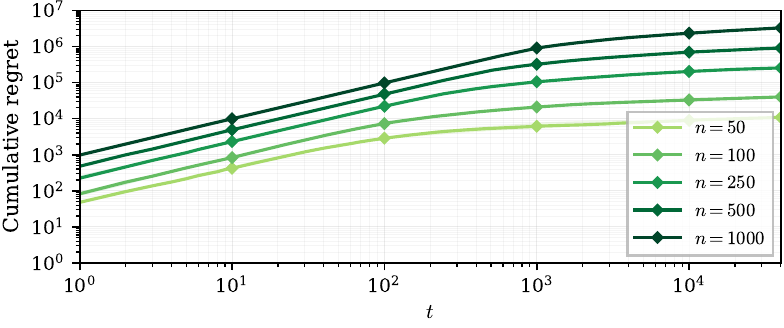}
\caption{Cumulative regret of Gibbs-TS under the linear-in-means reward model at network sizes $n \in \{50, 100, 250, 500, 1000\}$ over $T = 40{,}000$ rounds. Shaded bands denote 25th--75th percentile across 10 replications for each $n$.}
\label{fig:linmeans}
\end{figure}

\subsection{Tuning the ETC Isolation Budget $m$}
\label{app:etc_m_sweep}

The ETC variant's exploration phase length $m$ controls a bias-variance tradeoff: small $m$ produces a noisy graph estimate that may misclassify edges, while large $m$ pays a fixed exploration cost in regret regardless of subsequent learning. To assess whether ETC's regret in our main experiments reflects an under-tuned $m$, we sweep $m \in \{2, 5, 10, 15, 20, 30, 40, 60\}$ on the linear-in-means experiment of Section~\ref{sec:exp_real} ($T=10{,}000$, $\sigma=1.0$, 10 replications per cell with matched seeds across $m$). Table~\ref{tab:etc_m_sweep} reports median final regret and graph recovery for each $m$ on both real network topologies.

The optimal $m$ differs across the two networks: $m^*=10$ on village10 (median regret $64{,}365$) and $m^*=2$ on schid3 (median regret $78{,}513$). At every $m$ in the grid, ETC's regret remains an order of magnitude above the Gibbs-TS regret reported in Section~\ref{sec:exp_real} for the same networks. The schid3 curve is monotonically increasing in $m$, indicating that on the larger, sparser network the cost of dedicated exploration outweighs the benefit even at minimal isolation budgets. Network accuracy varies little across $m$ because the metric is dominated by correctly-classified non-edges in sparse graphs; the number of true edges identified grows with $m$, but so does Phase~1 regret.

\begin{table}[h]
\centering
\small
\caption{ETC tuning sweep over isolation budget $m$ on the linear-in-means real-network experiments ($T=10{,}000$, 10 matched seeds per cell). Median final regret with IQR. The optimal $m$ is network-dependent and ETC's regret remains well above Gibbs-TS at every grid point.}
\label{tab:etc_m_sweep}
\begin{tabular}{@{}lcccccccc@{}}
\toprule
$m$              & 2       & 5       & 10      & 15      & 20      & 30      & 40      & 60      \\
\midrule
\multicolumn{9}{l}{\emph{village10} ($n=63$, $|E|=114$, optimum at $m^*=10$)} \\
Median regret    & 73{,}206 & 70{,}536 & \textbf{64{,}365} & 75{,}408 & 69{,}455 & 76{,}557 & 85{,}610 & 109{,}864 \\
Edges recovered  & 4.5     & 6.0     & 8.0     & 10.0    & 11.5    & 16.5    & 18.5    & 22.5    \\
\midrule
\multicolumn{9}{l}{\emph{schid3} ($n=121$, $|E|=251$, optimum at $m^*=2$)} \\
Median regret    & \textbf{78{,}513} & 94{,}762 & 108{,}953 & 129{,}784 & 139{,}516 & 180{,}796 & 220{,}341 & 302{,}053 \\
Edges recovered  & 22.5    & 19.5    & 22.0    & 22.5    & 28.0    & 30.5    & 41.0    & 52.0    \\
\bottomrule
\end{tabular}
\end{table}

\subsection{Sweep-Count Ablation}\label{app:ablation}

Theorem~\ref{thm:modular_collapse} bounds the Bayesian regret of exact
posterior sampling, while Algorithm~\ref{alg:gibbs_ts} draws approximate
samples via $K$ Gibbs sweeps per round. We provide two empirical
diagnostics supporting $K=10$ as the paper default: a within-round
diagnostic on summary statistics of the chain state, and a sweep-count
ablation on regret.

\paragraph{Within-round chain behavior.}
At a single problem instance ($n=8$, count-based NIA reward,
Erd\H{o}s--R\'enyi graph with edge probability $0.3$), we run an extended
Gibbs chain of $K_{\text{diag}} = 200$ sweeps starting from the warm-start
state at snapshot rounds $t \in \{10, 50, 200, 1000\}$, across 75
independent seeds. Table~\ref{tab:mixing_drift} reports the drift between
the value at $k=10$ and the long-run average over $k > 100$, normalized by
cross-seed standard error. Two summary statistics of the chain state ---
squared Frobenius distance from the true adjacency and joint log-likelihood --- stabilize within the first 10
sweeps: at every snapshot round, the value at $k=10$ lies within one
cross-seed standard error of the plateau. Edge-flip activity in the second
half of each chain remains nonzero (0.19--0.53 flips per sweep depending
on $t$), with motion contracting as the posterior sharpens, confirming the
chain is actively moving rather than stuck.

\begin{table}[h]
\centering
\small
\caption{Drift between $K=10$ and the post-mixing plateau ($k > 100$),
expressed in cross-seed standard-error units, 75 seeds. Values $\leq 1$
indicate the $K=10$ sample is within sampling error of the plateau.}
\label{tab:mixing_drift}
\begin{tabular}{@{}lcccc@{}}
\toprule
Metric & $t=10$ & $t=50$ & $t=200$ & $t=1000$ \\
\midrule
$\|\mathbf{A}^{(k)} - \mathbf{A}^*\|_F^2$ & $0.37$ & $0.94$ & $0.08$ & $0.35$ \\
Log-likelihood                            & $0.10$ & $0.27$ & $0.23$ & $0.00$ \\
\bottomrule
\end{tabular}
\end{table}

\paragraph{Sweep-count ablation on regret.}
We run Gibbs-TS at $n=8$ on the linear-in-means reward with
$K \in \{1, 3, 5, 10, 20, 50\}$, $T = 2{,}000$, $B = 3$, $\sigma = 0.5$,
across 75 replications using a matched-seed protocol that shares an RNG
seed across $K$ values within each replication.

\begin{table}[h]
\centering
\small
\caption{Final cumulative regret at $T=2{,}000$ under the linear-in-means
reward, 75 replications, matched-seed protocol. Bootstrap 95\% CIs on the
mean reported in brackets; 10\%-trimmed mean reported as a robust
location summary.}
\label{tab:K_ablation}
\begin{tabular}{@{}lccccc@{}}
\toprule
$K$ & Median [IQR] & Mean (95\% CI) & Trim.\ mean & F1 (median) & Wall (ms) \\
\midrule
1   & 112 [80, 237]  & 205 [157, 264] & 157 & 0.84 & 1.9  \\
3   & 106 [77, 149]  & 134 [114, 155] & 116 & 0.88 & 2.4  \\
5   & 109 [74, 170]  & 178 [130, 239] & 124 & 0.87 & 3.4  \\
10  & 103 [70, 128]  & 150 [108, 212] & 104 & 0.88 & 5.9  \\
20  & 101 [73, 159]  & 172 [130, 226] & 123 & 0.86 & 10.9 \\
50  & 98  [70, 124]  & 118 [99, 140]  & 99  & 0.88 & 25.5 \\
\bottomrule
\end{tabular}
\end{table}

Median final regret varies across $K$ within a narrow band (98--112) that
is small relative to within-$K$ dispersion (IQRs of 50--160), and F1 is
essentially flat across $K$ (0.84--0.88). Paired comparisons against
$K=10$ are statistically indistinguishable from zero at every $K$: the
largest paired mean difference is $-32.2 \pm 28.9$ at $K=50$, comparable in magnitude to one standard error, and the fraction of replications in which a given
$K$ beats $K=10$ ranges from $30/75$ at $K=1$ to $42/75$ at $K=50$. The
regret distribution at all values of $K$ exhibits heavy right tails that
inflate the cross-rep mean and its standard error; the 10\%-trimmed mean,
robust to these tails, lies between 99 and 157 across $K$. Per-round
wall-clock cost grows linearly in $K$. We use $K=10$ as the paper default;
the ablation indicates that performance is not sensitive to this choice
within the range tested.

\section{Experiment Reproducibility Tables}

\subsection{Head-to-Head at $n=8$ (pairwise NIA, small and large $\xi$)}
\label{app:proto_head_to_head}

\begin{table}[H]
\centering
\small
\caption{Protocol for the Section~\ref{sec:exp_head_to_head} head-to-head
comparison. Both regimes share all settings except the $\xi$ block range.}
\begin{tabular}{@{}ll@{}}
\toprule
Setting & Value \\
\midrule
Network family & Erd\H{o}s--R\'enyi, $n = 8$, $p_{\text{edge}} = 0.3$, fresh per rep \\
Reward function & Pairwise NIA: $r_i = \mu_i Z_i + \sum_{j \in \mathcal{N}_i} \gamma_{ij} Z_j + \sum_{j<k \in \mathcal{N}_i} \xi_{ijk} Z_j Z_k + \varepsilon_i$ \\
$\mu_i$ distribution & $\mathrm{Uniform}[0.5, 1.5]$ iid per rep \\
$\gamma_{ij}$ distribution & $\mathrm{Uniform}[0.3, 1.0]$ iid per rep \\
$\xi_{ijk}$ distribution (small $\xi$) & $\mathrm{Uniform}[-0.4, 0.4]$ iid per rep \\
$\xi_{ijk}$ distribution (large $\xi$) & $\mathrm{Uniform}[-3.0, 3.0]$ iid per rep \\
Horizon $T$ & $2{,}000$ \\
Budget $B$ & $3$ \\
Noise $\sigma$ & $0.5$ \\
Replications & $30$ (matched seeds; $\text{seed}_0 = 1000$) \\
Algorithms & Gibbs-TS (well-spec.), Gibbs-TS (additive misspec., drops $\xi$), \\
 & ETC-TS (well-spec.), \citet{agarwal2024} \\
Gibbs sweeps $K$ & $10$  \\
Edge prior $\rho$ & $0.3$ (= $p_{\text{edge}}$) \\
Prior mean $\boldsymbol{\mu}_0$ & $\mathbf{0}$ \\
Prior covariance $\boldsymbol{\Sigma}_0$ & $10 \cdot \mathbf{I}$ \\
Phase-2 ridge for ETC ($1/\lambda$) & $\lambda = 0.1$, equivalent to $\boldsymbol{\Sigma}_0 = 10 \cdot \mathbf{I}$ \\
ETC isolation budget $m$ & $20$ rounds per node \\
ETC edge threshold & $\delta_\gamma / 2$ with $\delta_\gamma = 0.3$ \\
Agarwal et al. regularization & Default Lasso, CV-selected (their published codebase) \\
Treatment-selection optimizer & Brute-force exact enumeration over $\binom{8}{3}$ subsets \\
\bottomrule
\end{tabular}
\end{table}

\subsection{Real-Network Experiments (village~10 and SCHID~3, linear-in-means)}
\label{app:proto_real}

\begin{table}[H]
\centering
\small
\caption{Protocol for the Section~\ref{sec:exp_real} real-network experiments.
Both networks share the reward parameterization, sampler hyperparameters, and
horizon; only the network and $B$ differ.}
\begin{tabular}{@{}ll@{}}
\toprule
Setting & Value \\
\midrule
Network 1 (village~10) & \citet{banerjee_diffusion2013} borrowing-money household \\
 & $n = 63$, $|E| = 114$, density $0.058$, isolates dropped \\
Network 2 (SCHID~3) & \citet{paluck_shepherd_aranow2016} friendship roster, symmetrized \\
 & $n = 121$, $|E| = 251$, density $0.035$, isolates dropped \\
Reward function & Per-node linear-in-means: $r_i = \mu_i Z_i + \beta_i \bar{Z}_{\mathcal{N}_i} + \varepsilon_i$ \\
$\mu_i, \beta_i$ distribution & $\mathrm{Uniform}[0,1]$ iid per node, fresh per rep \\
Horizon $T$ & $10{,}000$ \\
Budget $B$ & $\lceil n/5 \rceil$: $13$ (village), $25$ (SCHID~3) \\
Noise $\sigma$ & $1.0$ \\
Replications & $30$ per network \\
Algorithms & Gibbs-TS (linear-means variant); Known-$\mathbf{A}$ TS oracle \\
Gibbs sweeps $K$ & $10$ \\
Edge prior $\rho$ & $1/n$ (= $0.0159$ village, $0.00826$ SCHID~3) \\
Prior mean $\boldsymbol{\mu}_0$ & $\mathbf{0}$ (per-node $(\mu_i, \beta_i)$) \\
Prior covariance $\boldsymbol{\Sigma}_0$ & $10 \cdot \mathbf{I}$ (per-node $2 \times 2$ block) \\
Posterior decomposition & Per-node $2 \times 2$ block-diagonal $\theta$-sampler \\
Treatment selection & Top-$B$ by expected marginal reward (no ILP solver) \\
Graph-accuracy snapshots & Every 100 rounds \\
Per-rep checkpointing & SLURM array tasks (one rep per task) \\
\bottomrule
\end{tabular}
\end{table}

\subsection{Downstream Causal Estimation: $n=20$ SBM, count-based NIA}
\label{app:proto_downstream_n20}

\begin{table}[H]
\centering
\small
\caption{Protocol for the $n=20$ SBM downstream-causal panel
(Section~\ref{sec:exp_downstream}, Table~\ref{tab:downstream}, left). The
design phase reuses the appendix count-based NIA sweep at $n=20$
(Appendix~\ref{app:proto_sania_scaling}); only the inference-phase
configuration is reported here.}
\begin{tabular}{@{}ll@{}}
\toprule
Setting & Value \\
\midrule
Network family & SBM, $K = 2$ groups, $p_{\text{within}} = 0.25$, $p_{\text{between}} = 1/n = 0.05$ \\
Network size $n$ & $20$ \\
Reward (design \& inference) & Count-based NIA: $r_i = \mu Z_i + \sum_{k=1}^{d_{\max}} \gamma_k \mathbf{1}\{d^1_i = k\}$, $d_{\max} = 4$ \\
Parameter sampling & $\mu \sim \mathcal{N}(1, 0.2)$; $\gamma_k \sim \mathcal{N}(k, 0.5)$ iid per rep \\
Design-phase horizon $T$ & $2{,}000$ \\
Design-phase budget $B$ & $6$ ($= \lfloor n/3 \rfloor$) \\
Inference-phase horizon $T_{\text{eval}}$ & $2{,}000$ \\
Inference-phase design & $Z_{t,i} \overset{\text{iid}}{\sim} \mathrm{Bernoulli}(B/n) = \mathrm{Bernoulli}(0.30)$ \\
Noise $\sigma$ & $1.0$ \\
Replications & $15$ \\
Gibbs sweeps $K$ & $10$ \\
Edge prior $\rho$ & $1/n = 0.05$ \\
Prior mean $\boldsymbol{\mu}_0$ & $\mathbf{0}$ \\
Prior covariance $\boldsymbol{\Sigma}_0$ & $10 \cdot \mathbf{I}$ \\
Warmup rounds $T_{\text{warm}}$ & $20$ (random-exploration warmup) \\
Graph estimate $\hat{\mathbf{A}}$ & Marginal posterior mean at $t = T$, thresholded at $0.5$ \\
Estimators & (i) Gibbs posterior mean of $\boldsymbol{\theta}$ from adaptive design $\mid \hat{\mathbf{A}}$; \\
 & (ii) OLS on inference phase $\mid \hat{\mathbf{A}}$ (ridge fallback if $X^\top X$ ill-conditioned, $\lambda = 0.01$); \\
 & (iii) OLS on inference phase $\mid \mathbf{A}$ (oracle) \\
Treatment-selection optimizer (design) & Gurobi ILP \\
\bottomrule
\end{tabular}
\end{table}

\subsection{Downstream Causal Estimation: village~10, linear-in-means}
\label{app:proto_downstream_village}

\begin{table}[H]
\centering
\small
\caption{Protocol for the village~10 downstream-causal panel
(Section~\ref{sec:exp_downstream}, Table~\ref{tab:downstream}, right).
The design phase is the village run from Section~\ref{sec:exp_real};
only the inference-phase configuration is reported here.}
\begin{tabular}{@{}ll@{}}
\toprule
Setting & Value \\
\midrule
Network & village~10 borrowmoney\_HH (\citealp{banerjee_diffusion2013}), \\
 & $n = 63$, $|E| = 114$, density $0.058$, isolates dropped \\
Reward (design \& inference) & Per-node linear-in-means: $r_i = \mu_i Z_i + \beta_i \bar{Z}_{\mathcal{N}_i} + \varepsilon_i$ \\
$\mu_i, \beta_i$ distribution & $\mathrm{Uniform}[0,1]$ iid per node, fresh per rep \\
Design-phase horizon $T$ & $10{,}000$ \\
Design-phase budget $B$ & $13$ ($= \lceil n/5 \rceil$, default in the design driver) \\
Inference-phase horizon $T_{\text{eval}}$ & $10{,}000$ \\
Inference-phase budget $B$ & $21$ ($= \lfloor n/3 \rfloor$) \\
Inference-phase design & $Z_{t,i} \overset{\text{iid}}{\sim} \mathrm{Bernoulli}(B/n)$, with $B$ as above \\
Noise $\sigma$ & $1.0$ \\
Replications & $10$ \\
Gibbs sweeps $K$ & $10$ \\
Edge prior $\rho$ & $1/n = 0.0159$ \\
Prior mean $\boldsymbol{\mu}_0$ & $\mathbf{0}$ \\
Prior covariance $\boldsymbol{\Sigma}_0$ & $10 \cdot \mathbf{I}$ (per-node $2 \times 2$ block) \\
Posterior decomposition (design) & Per-node $2 \times 2$ block-diagonal $\theta$-sampler \\
Graph estimate $\hat{\mathbf{A}}$ & Marginal posterior mean at $t = T$, thresholded at $0.5$ \\
Estimators & (i) Per-node Gibbs posterior mean of $(\mu_i, \beta_i)$ $\mid \hat{\mathbf{A}}$; \\
 & (ii) Per-node OLS on inference phase $\mid \hat{\mathbf{A}}$ (ridge $\lambda = 0.01$ fallback); \\
 & (iii) Per-node OLS on inference phase $\mid \mathbf{A}$ (oracle) \\
Treatment selection (design) & Top-$B$ by expected marginal reward \\
\bottomrule
\end{tabular}
\end{table}

\subsection{NIA Spec~A and Spec~B at $n=8$}
\label{app:proto_nia_specs}

\begin{table}[H]
\centering
\small
\caption{Protocol for the Appendix~\ref{app:nia_specs} comparisons. Spec~A
(saturation) and Spec~B (treatment--spillover interaction) share all settings
except the reward form and the $\lambda_i$ block in Spec~B.}
\begin{tabular}{@{}ll@{}}
\toprule
Setting & Value \\
\midrule
Network family & Erd\H{o}s--R\'enyi, $n = 8$, $p_{\text{edge}} = 0.3$, fresh per rep \\
Reward (Spec~A, saturation) & $r_i = \mu_i Z_i \mathbf{1}\{d^1_i = 0\} + \sum_{j \in \mathcal{N}_i} \gamma_{ij} Z_j + \varepsilon_i$ \\
Reward (Spec~B, interaction) & $r_i = \mu_i Z_i + \sum_{j \in \mathcal{N}_i} \gamma_{ij} Z_j + \lambda_i Z_i d^1_i + \varepsilon_i$ \\
$\mu_i$ distribution & $\mathrm{Uniform}[0.5, 1.5]$ iid per rep \\
$\gamma_{ij}$ distribution & $\mathrm{Uniform}[0.3, 1.0]$ iid per rep \\
$\lambda_i$ distribution (Spec~B only) & $\mathrm{Uniform}[-0.3, 0.3]$ iid per rep \\
Horizon $T$ & $2{,}000$ \\
Budget $B$ & $3$ \\
Noise $\sigma$ & $0.5$ \\
Replications & $30$ ($\text{seed}_0 = 1000$) \\
Algorithms & Gibbs-TS, ETC-TS, \citet{agarwal2024} \\
Gibbs sweeps $K = 10$ \\
Edge prior $\rho$ & $0.3$ \\
Prior mean $\boldsymbol{\mu}_0$ & $\mathbf{0}$ \\
Prior covariance $\boldsymbol{\Sigma}_0$ & $10 \cdot \mathbf{I}$ \\
Phase-2 ridge for ETC ($1/\lambda$) & $\lambda = 0.1$, equivalent to $\boldsymbol{\Sigma}_0 = 10 \cdot \mathbf{I}$ \\
ETC isolation budget $m$ & $20$ rounds per node \\
ETC edge threshold & $\delta_\gamma / 2$ with $\delta_\gamma = 0.3$ \\
Agarwal et al. regularization & Default Lasso, CV-selected \\
Treatment-selection optimizer & Gurobi ILP \\
\bottomrule
\end{tabular}
\end{table}

\subsection{Robustness to Reward Model Misspecification}
\label{app:proto_misspec}

\begin{table}[H]
\centering
\small
\caption{Protocol for the Appendix~\ref{app:misspec} misspecification
experiment. Two data-generating processes (GP1 = Spec~A, GP2 = Spec~B);
Gibbs-TS and ETC-TS fit the misspecified count-based NIA model, while
\citet{agarwal2024} runs correctly specified.}
\begin{tabular}{@{}ll@{}}
\toprule
Setting & Value \\
\midrule
Network family & Erd\H{o}s--R\'enyi, $n = 8$, $p_{\text{edge}} = 0.3$, fresh per rep \\
Data-generating process GP1 & Spec~A (saturation) -- see Appendix~\ref{app:proto_nia_specs} \\
Data-generating process GP2 & Spec~B (treatment--spillover interaction) \\
Fit model (Gibbs-TS, ETC-TS) & Count-based NIA: $r_i = \mu_i Z_i + \sum_k \gamma_{i,k} \mathbf{1}\{d^1_i = k\}$, $d_{\max} = n - 1 = 7$ \\
Fit model (Agarwal) & Fourier basis over the full $2^n$ exposure space (correctly specified) \\
$\mu_i, \gamma_{ij}$ (DGP) & $\mathrm{Uniform}[0.5, 1.5]$ and $\mathrm{Uniform}[0.3, 1.0]$ \\
$\lambda_i$ (DGP, GP2 only) & $\mathrm{Uniform}[-0.3, 0.3]$ \\
Horizon $T$ & $2{,}000$ \\
Budget $B$ & $3$ \\
Noise $\sigma$ & $0.5$ \\
Replications & $30$ per DGP \\
Gibbs sweeps $K$ & $K_{\text{gibbs}} = 10$ \\
Edge prior $\rho$ & $0.3$ \\
Prior mean $\boldsymbol{\mu}_0$ & $\mathbf{0}$ \\
Prior covariance $\boldsymbol{\Sigma}_0$ & $10 \cdot \mathbf{I}$ \\
ETC isolation budget $m$ & $20$ rounds per node \\
ETC edge threshold & $\delta_\gamma / 2$ with $\delta_\gamma = 0.3$ \\
ETC Phase-2 ridge & $\lambda = 0.1$ ($\boldsymbol{\Sigma}_0 = 10 \cdot \mathbf{I}$) \\
Treatment-selection optimizer & Gurobi ILP \\
\bottomrule
\end{tabular}
\end{table}

\subsection{Sensitivity to the Edge Prior $\rho$}
\label{app:proto_rho}

\begin{table}[H]
\centering
\small
\caption{Protocol for the Appendix~\ref{app:rho_sensitivity} edge-prior
sensitivity sweep. Reuses the small-$\xi$ setting from
Appendix~\ref{app:proto_head_to_head} and varies only $\rho$.}
\begin{tabular}{@{}ll@{}}
\toprule
Setting & Value \\
\midrule
Network family & Erd\H{o}s--R\'enyi, $n = 8$, $p_{\text{edge}} = 0.3$, fresh per seed \\
Reward function & Pairwise NIA, small $\xi$ regime \\
$\mu_i$ distribution & $\mathrm{Uniform}[0.5, 1.5]$ \\
$\gamma_{ij}$ distribution & $\mathrm{Uniform}[0.3, 1.0]$ \\
$\xi_{ijk}$ distribution & $\mathrm{Uniform}[-0.4, 0.4]$ \\
Horizon $T$ & $2{,}000$ \\
Budget $B$ & $3$ \\
Noise $\sigma$ & $0.5$ \\
Algorithm & Gibbs-TS (well-specified fit) \\
Swept variable & Edge prior $\rho$ \\
Grid for $\rho$ & $\{0.05,\ 0.15,\ 0.30,\ 0.50,\ 0.70\}$ \\
Replications per $\rho$ & $15$ matched seeds (seed range $1000$--$1014$) \\
Matching protocol & $A_{\text{true}}$, $\theta_{\text{true}}$, oracle, and bandit RNGs depend only on \\
 & seed (not $\rho$); $\rho$ enters only the edge-prior log-prior in Gibbs sweeps \\
Gibbs sweeps $K$ & $10$ (flat) \\
Prior mean $\boldsymbol{\mu}_0$ & $\mathbf{0}$ \\
Prior covariance $\boldsymbol{\Sigma}_0$ & $10 \cdot \mathbf{I}$ \\
Post-hoc edge-marginal sweeps & $50$ extra Gibbs sweeps from $A^{(T)}$ for marginal $\hat P(A_{ij}=1 \mid \text{data})$ \\
Treatment-selection optimizer & Brute-force exact enumeration \\
\bottomrule
\end{tabular}
\end{table}

\subsection{Count-Based NIA at $n \in \{20, 40\}$}
\label{app:proto_sania_scaling}

\begin{table}[H]
\centering
\small
\caption{Protocol for the Appendix~\ref{sec:exp_sania_scaling} scaling
experiment. Both network sizes share the reward parameterization, sampler
hyperparameters, and SBM construction; only $n$, $T$, $B$, $T_{\text{warm}}$,
and replication count differ.}
\begin{tabular}{@{}ll@{}}
\toprule
Setting & Value \\
\midrule
Network family & SBM, $K = \max(2, n/10)$ groups; $p_{\text{within}} = 0.25$, $p_{\text{between}} = 1/n$ \\
SBM at $n = 20$ & $K = 2$, $p_{\text{between}} = 0.05$ \\
SBM at $n = 40$ & $K = 4$, $p_{\text{between}} = 0.025$ \\
Reward function & Count-based NIA with shared parameters: \\
 & $r_i = \mu Z_i + \sum_{k=1}^{d_{\max}} \gamma_k \mathbf{1}\{d^1_i = k\} + \varepsilon_i$, \quad $d_{\max} = 4$ \\
Parameter sampling& $\mu \sim \mathcal{N}(1.0,\ 0.2)$; $\gamma_k \sim \mathcal{N}(k,\ 0.5)$ iid per rep \\
Horizon $T$ & $2{,}000$ ($n=20$); \quad $5{,}000$ ($n=40$) \\
Budget $B$ & $6$ ($n=20$); \quad $13$ ($n=40$); both $\lfloor n/3 \rfloor$ in the runner \\
Noise $\sigma$ & $1.0$ \\
Replications & $30$ ($n=20$); \quad $10$ ($n=40$) \\
Algorithms & Gibbs-TS, ETC-TS, Known-$\mathbf{A}$ TS \\
Gibbs sweeps $K$ & $10$ (flat) \\
Edge prior $\rho$ & $1/n$ ($0.05$ at $n=20$, $0.025$ at $n=40$) \\
Prior mean $\boldsymbol{\mu}_0$ & $\mathbf{0}$ \\
Prior covariance $\boldsymbol{\Sigma}_0$ & $10 \cdot \mathbf{I}$ \\
Warmup rounds $T_{\text{warm}}$ & $20$ ($n=20$); \quad $40$ ($n=40$); random-exploration warmup \\
ETC isolation budget $m$ & $40$ rounds per node \\
ETC edge threshold & $\delta_\gamma / 2$ with $\delta_\gamma = 1.0$ \\
Graph-accuracy snapshots & Every $10$ rounds \\
Treatment-selection optimizer & Gurobi ILP \\
\bottomrule
\end{tabular}
\end{table}

\subsection{Linear-in-Means Scaling at $n \in \{50, 100, 250, 500, 1000\}$}
\label{app:proto_linmeans_scaling}

\begin{table}[H]
\centering
\small
\caption{Protocol for the Appendix~\ref{app:linmeans_scaling} scaling study.}
\begin{tabular}{@{}ll@{}}
\toprule
Setting & Value \\
\midrule
Network family & SBM, $K = n/10$ groups, $p_{\text{within}} = 0.30$, $p_{\text{between}} = 1/n$ \\
Network sizes $n$ & $\{50,\ 100,\ 250,\ 500,\ 1000\}$ \\
Reward function & Linear-in-means: $r_i = \mu_i Z_i + \beta_i \bar{Z}_{\mathcal{N}_i} + \varepsilon_i$ \\
Parameter sampling & Single $\mu, \beta$ shared across nodes per replication, with \\
 & $\mu \sim \mathcal{N}(2,\ 1)$, $\beta \sim \mathcal{N}(1,\ 0.5)$ \\
Horizon $T$ & $40{,}000$ \\
Budget $B$ & $n / 5$: $\{10, 20, 50, 100, 200\}$ \\
Noise $\sigma$ & $1.0$ \\
Prior variance $\tau^2$ on $(\mu, \beta)$ & $1.0$ (Linear-Means TS posterior) \\
Replications & $10$  \\
Algorithms & Linear-Means TS (Gibbs-TS variant); Known-$\mathbf{A}$ TS oracle \\
Warmup rounds $T_{\text{warm}}$ & $0$ \\
Treatment selection & Top-$B$ by expected marginal reward (no ILP solver) \\
Base seed & $42$; rep seed = $42 + 1000n + 10r$ \\
\bottomrule
\end{tabular}
\end{table}

\subsection{Tuning the ETC Isolation Budget $m$}
\label{app:proto_etc_m}

\begin{table}[H]
\centering
\small
\caption{Protocol for the Appendix~\ref{app:etc_m_sweep} ETC tuning sweep.
Reuses the village~10 / SCHID~3 linear-in-means configuration from
Appendix~\ref{app:proto_real}; ETC is the only algorithm.}
\begin{tabular}{@{}ll@{}}
\toprule
Setting & Value \\
\midrule
Networks & village~10 ($n=63$, $|E|=114$) and SCHID~3 ($n=121$, $|E|=251$) \\
Reward function & Per-node linear-in-means (same as Section~\ref{sec:exp_real}) \\
$\mu_i, \beta_i$ distribution & $\mathrm{Uniform}[0,1]$ iid per node, fresh per rep \\
Horizon $T$ & $10{,}000$ \\
Budget $B$ & $\lceil n/5 \rceil$: $13$ (village), $25$ (SCHID~3) \\
Noise $\sigma$ & $1.0$ \\
Algorithm & ETC-into-TS (linear-means variant) \\
Swept variable & Phase-1 isolation budget $m$ (rounds per node) \\
Grid for $m$ & $\{2,\ 5,\ 10,\ 15,\ 20,\ 30,\ 40,\ 60\}$ \\
Replications per cell & $10$, matched seeds across $m$ within each network \\
ETC edge threshold $\tau$ & $3 \sigma \sqrt{2/m}$ on $\bar r_{i,j} - \bar r_{i,i}$, max-symmetrized \\
Phase-2 algorithm & Thompson sampling under $\hat{\mathbf{A}}$ (Linear-Means TS) \\
Phase-2 prior $\boldsymbol{\Sigma}_0$ & $10 \cdot \mathbf{I}$ \\
Treatment selection (Phase 2) & Top-$B$ by expected marginal reward \\
\bottomrule
\end{tabular}
\end{table}

\subsection{Sweep-Count Ablation}
\label{app:proto_K_ablation}

\begin{table}[H]
\centering
\small
\caption{Protocol for the Appendix~\ref{app:ablation} sweep-count ablation
on regret. (The within-round mixing diagnostic in
Table~\ref{tab:mixing_drift} runs $K_{\text{diag}} = 200$ Gibbs sweeps from
the warm-start state at four snapshot rounds with $75$ seeds; otherwise the
problem instance matches the line below.)}
\begin{tabular}{@{}ll@{}}
\toprule
Setting & Value \\
\midrule
Network family & Erd\H{o}s--R\'enyi, $n = 8$, $p_{\text{edge}} = 0.3$, fresh per rep \\
Reward function & Per-node linear-in-means: $r_i = \mu_i Z_i + \beta_i \bar{Z}_{\mathcal{N}_i} + \varepsilon_i$ \\
$\mu_i, \beta_i$ distribution & $\mathrm{Uniform}[0,1]$ iid per node, fresh per rep \\
Horizon $T$ & $2{,}000$ \\
Budget $B$ & $3$ \\
Noise $\sigma$ & $0.5$ \\
Algorithm & Gibbs-TS (linear-means variant) \\
Swept variable & Gibbs sweep count $K$ \\
Grid for $K$ & $\{1,\ 3,\ 5,\ 10,\ 20,\ 50\}$ \\
Replications & $75$ total ($15$ at $\text{seed}_0 = 2000$ + $60$ at $\text{seed}_0 = 3000$) \\
Matching protocol & Within each rep, the same RNG seed is used across $K$ values; \\
 & cross-$K$ differences are driven only by the extra Gibbs sweeps \\
Edge prior $\rho$ & $1/n = 0.125$ \\
Prior mean $\boldsymbol{\mu}_0$ & $\mathbf{0}$ \\
Prior covariance $\boldsymbol{\Sigma}_0$ & $10 \cdot \mathbf{I}$ (per-node $2 \times 2$ block) \\
Treatment selection & Top-$B$ by expected marginal reward (no Gurobi) \\
F1 metric & Edge-level F1 of $\hat{\mathbf{A}}$ at $t = T$ vs.\ $\mathbf{A}$ on the upper triangle \\
\bottomrule
\end{tabular}
\end{table}

\section{Computing Resources}
Figures were produced on a high-performance computing cluster using 50-75 CPUs in parallel, each assigned 1GB of RAM. The CPUs used were Intel(R) Xeon(R) Gold 6336Y. Gurobi software was used for all simulations. The authors received an academic license for free by registering through the Gurobi website.

\end{document}